\documentclass{amsart}

\usepackage{mathrsfs}
\usepackage{kbordermatrix}
\usepackage{enumerate}
\usepackage{algorithm}
\usepackage{algorithmic}
\usepackage{multirow}
\usepackage{xcolor}
\usepackage{lscape}

\usepackage{graphicx}
\usepackage{amsmath,amsfonts,amsthm,amsbsy,amssymb}
\usepackage{fixmath}
\usepackage{amssymb,latexsym}
\usepackage{skak}

\def\bm{\pmb{m}}

\def\bp{\pmb{p}}
\def\bq{\pmb{q}}
\def\br{\pmb{r}}

\def\bu{\pmb{u}}

\def\bw{\pmb{w}}
\def\bx{\pmb{x}}
\def\by{\pmb{y}}
\def\bz{\pmb{z}}

\def\bone{\pmb{1}}

\def\bbO{\mathbb{O}}
\def\bbR{\mathbb{R}}

\def\cR{\mathcal{R}}

\def\wtd{\widetilde}

\DeclareMathOperator{\diag}{diag}

\DeclareMathOperator{\kry}{kry}
\DeclareMathOperator{\opt}{opt}
\DeclareMathOperator{\rank}{rank}
\DeclareMathOperator{\raw}{raw}
\DeclareMathOperator{\tr}{tr}

\DeclareMathOperator{\st}{s.t.}

\DeclareMathOperator{\T}{T}

\def\scrA{\mathscr{A}}
\def\scrB{\mathscr{B}}

\def\scrX{\mathscr{X}}

\def\algres{{\tt res}}
\def\tol{{\tt tol}}
\def\nkry{n_{\kry}}

\newtheorem{theorem}{Theorem}
\newtheorem{lemma}{Lemma}[section]

\theoremstyle{definition}

\title[OMvA]{Orthogonal Multi-view Analysis by Successive Approximations via Eigenvectors}
\author[L. Wang] {Li Wang}
\author[L. Zhang]{Lei-Hong~Zhang}
\author[C. Shen]{Chungen~Shen}
\author[R. Li]{Ren-Cang~Li}
\thanks{
Li Wang is with Department of Mathematics and
Department of Computer Science and Engineering,
University of Texas at Arlington, Arlington, TX 76019-0408, USA. Email: li.wang@uta.edu. Corresponding Author.\\
Lei-Hong Zhang is with School of Mathematical Sciences and Institute of Computational Science, Soochow University, Suzhou 215006, Jiangsu, China. Email: longzlh@suda.edu.cn.\\
Chungen~Shen is with College of Science, University of Shanghai for Science and Technology, Shanghai 200093, China. Email: shenchungen@usst.edu.cn.\\
Ren-Cang Li is with Department of Mathematics,
University of Texas at Arlington, Arlington, TX 76019-0408, USA.
Email: rcli@uta.edu.
}

\begin{document}
\maketitle

\begin{abstract}
We propose a unified framework  for multi-view subspace learning to learn individual orthogonal projections for all views. The framework  integrates the correlations within multiple views, supervised discriminant capacity, and distance preservation in a concise and compact way.
It not only includes several existing models as special cases, but also inspires new novel models. To demonstrate its versatility to handle different learning scenarios, we showcase three new multi-view discriminant analysis models and two new multi-view multi-label classification ones under this framework. An efficient numerical method based on successive approximations via eigenvectors is presented to solve the associated optimization problem.
The method is built upon an iterative Krylov subspace method which can easily scale up for high-dimensional datasets.  Extensive experiments are conducted on various real-world datasets for multi-view discriminant analysis and  multi-view multi-label classification. The experimental results demonstrate that the proposed models
are consistently competitive to and often better than the compared methods that do not learn orthogonal projections.
\end{abstract}

\section{Introduction} Multi-view data are increasingly collected for a variety of applications in the real world. They usually contain complementary, redundant, and corroborative contents and so are more informative than single-view data when it comes
to characterize objects of the real-world.
It is rather natural for human beings to perceive the world through comprehensive information collected by multiple sensory organs, but it is an open question on how to endow machines with analogous cognitive capabilities to do the same.
One of the fundamental challenges is how to represent and summarize multi-view data in such a way
that comprehensive information concealed in multi-view data can be properly exploited by multi-view learning models.

The heterogeneity gap \cite{peng2019cm} among multiple views makes it challenging to construct such representations since features extracted from different views with similar semantics may be located in completely different subspaces, e.g., text is often symbolic while audio and image are signals.
A significant research effort has been about narrowing this gap by seeking a common semantic subspace into which  the heterogeneous features from different views are projected.

Multi-view subspace learning, as the most popularly studied methodology for multi-view learning \cite{zhao2017multi,xu2013survey},
aims to narrow the heterogeneity gap under the assumption that
all views are generated
from a common latent space via some unknown transformations in the first place.  The most representative subspace learning model is the canonical correlation analysis (CCA)  \cite{hote:1936}, which was originally proposed to learn two linear projections by maximizing the cross-correlation between two views in a common space. It has since been extended to more than two views \cite{nielsen2002multiset}, nonlinear projections via either kernel representation \cite{hardoon2004canonical} or deep representation \cite{andrew2013deep},  supervised learning \cite{sharma2012generalized,cao2017generalized,sun2015multiview}, and multi-output learning such as multi-label classification \cite{sun2010canonical} and multi-target regression \cite{wang2020orthogonal}.

Recent researches have demonstrated
that orthogonality built  into single-view subspace learning models admits desirable advantages
such as more noise-tolerant, better suited for data visualization and distance preservation \cite{jolliffe1986principal,ye2005characterization,wang2007trace,zhln:2010,kokiopoulou2007orthogonal,cai2005orthogonal}.
An orthogonal projection preserves the pairwise distance so long as the vectors to be projected live
in the range of the projection.  Distance preservation, as one of the most important learning criteria, has also been successfully
demonstrated in  learning methods such as kernel learning \cite{weinberger2004learning} and density estimation \cite{wang2017latent}.

Orthogonality has been successfully explored in multi-view subspace learning,  including orthogonal CCA (OCCA) \cite{wang2020orthogonal,shsy:2013,cugh:2015,Zhang2020Self},  orthogonal multiset CCA (OMCCA) \cite{Zhang2020Self,shen2015orthogonal}, and  multi-view partial least squares (PLS) \cite{wang2020scalable}. However, most multi-view subspace learning methods stay clear from  orthogonality constraints for two technical obstacles:
\begin{enumerate}[1)]
	\item adding orthogonality constraints may cause incompatibility to inherent constraints already there in existing models, and
	\item even if there is no incompatibility issue, the resulting optimization problem is generally hard to solve. 
\end{enumerate}
Generic optimization methods are often too slow even for datasets of modest scale and practically infeasible for high dimensional data. As a result, most existing learning methods \cite{sharma2012generalized,cao2017generalized,hu2019multi} resort to solving certain related relaxed problems of their original formulations as generalized eigenvalue problems, for which well-developed numerical linear algebra techniques can be readily deployed to handle high-dimensional datasets
but at a price of degrading learning performance.

This issue has been previously  studied in the case of the trace ratio formulation {\em vs.} the ratio trace formulation
for single-view dimensionality reduction  in the context of linear discriminant analysis (LDA). Authors in \cite{wang2007trace} argued that the trace ratio formulation with the orthogonality constraint is essential and
can lead to superiority over the ratio trace formulation which is a relaxation
of the trace ratio formulation as a generalized eigenvalue problem.
The cross-correlation between two views in CCA is inherently defined as a trace ratio formulation \cite{hote:1936}. Moreover, the objective function of the trace ratio formulation is invariant under any orthogonal transformation, which is more beneficial to classification and clustering in the reduced space than the ratio trace formulation that is invariant under any non-singular transformation. This motivates the study of orthogonal LDA (OLDA) \cite{wang2007trace,zhln:2010} and orthogonal CCA (OCCA) \cite{cugh:2015}. However, no orthogonal extension to supervised multi-view subspace learning has yet been explored.

Our goals in this paper are twofold. We will propose a unified framework for orthogonal multi-view analysis
to resolve the obstacle 1) we previously pointed out. Specifically, we take the trace ratio formulation to model the pairwise correlations of multiple views by strictly following their original definitions. Various supervised information can be incorporated into the numerators or denominators of the trace ratios  in order to capture the class separability or coherence. Orthogonality constraints are added without causing any incompatibility issue. All three ingredients are integrated together in a concise and consistent way by the proposed framework. However,
the resulting optimization problem  is a  challenging one. 
That is the obstacle 2) we previously pointed out. Instead of solving the optimization problem
as it is, we propose an efficient optimization method called orthogonal successive approximation via eigenvectors (OSAVE) to calculate an approximate solution. 

\noindent\textbf{Contributions.} The main contributions of this paper are summarized as follows:
\begin{itemize}
	\item We propose a unified multi-view subspace learning framework, which can naturally integrate the dependency among multiple views, supervised information, and simultaneously learn orthogonal projections in a concise and compact formulation. OLDA, OCCA and OMCCA are special cases of the proposed framework.
	
	\item Our framework can be flexibly adapted for various learning scenarios. To justify the flexibility, we instantiate several new models from the proposed framework.
	Three models are proposed for multi-view feature extraction, and two models  for multi-view multi-label classification. Different from existing ones, our models are directly built on the essential trace ratio formulation with orthogonality
	constraints.
	
	\item To solve the challenging optimization problem of the proposed framework, we present a successive approximation algorithm, which is built upon well-developed numerical linear algebra techniques. We describe an iterative Krylov subspace method for calculating the top eigenvector of generalized eigenvalue problem
	$A\bx=\lambda B\bx$ with possibly a singular $B$. The Krylov subspace method can serve as the workhorse for scalability.
	
	\item Extensive experiments are conducted for evaluating the proposed models against existing learning methods in terms of two learning tasks: multi-view feature extraction and multi-view multi-label classification. Experimental results  on various real-world datasets demonstrate that our proposed models perform competitively to and often better than baselines.
\end{itemize}

\noindent\textbf{Paper organization.} We first describe the background of this work from single- and multi-view feature extraction and briefly review the relevant existing models in Section \ref{sec:related-work}. In Section \ref{sec:omva}, we propose the novel unified framework for orthogonal multi-view analysis, and their instantiated models for multi-view discriminant analysis and multi-view multi-label classification. The proposed successive approximation algorithm is presented in Section \ref{sec:alg4OMDA} with its key component in Section \ref{sec:geig}. Extensive experiments are conducted in Section \ref{sec:experiments}. Finally, we draw our conclusions in Section \ref{sec:conclusions}.

\noindent\textbf{Notation.}
$\bbR^{m\times n}$ is the set of $m\times n$ real matrices and $\bbR^n=\bbR^{n\times 1}$. $I_n\in\bbR^{n\times n}$ is
the identity matrix of size $n \times n$, and $\bone_n\in\bbR^n$ is the vector of all ones. $\|\bx\|_2$
is the 2-norm of a vector $\bx\in \bbR^n$. For $B\in\bbR^{m\times n}$,
$\cR(B)$ is the column subspace.
$B\succ 0 (\succeq 0)$ means that $B$ is symmetric positive definite (semi-definite). 
The Stiefel manifold
\begin{align}
\bbO^{n\times k}=\{X\in\bbR^{n\times k}\,:\, X^{\T}X=I_k\}
\end{align}
is an embedded submanifold of $\bbR^{n\times k}$ endowed with the standard inner product $\langle X,Y \rangle=\tr(X^{\T}Y)$ for $X,Y\in \bbR^{n\times k}$,
where $\tr(X^{\T}Y)$ is the trace of $X^{\T}Y$.

\section{Background and Related Work} \label{sec:related-work}
Feature extraction is an important tool for multivariate data analysis. A number of methods have been proposed in the literature. In what follows, we will review a large family of feature extraction methods for learning linear transformations and then explain their characteristics.

\subsection{Problem Setup}
We start by a general setup for  feature extraction learning on
data of multiple views  and their class labels, and then explain their  representations
in a common space.

Let $\{ (\bx_i^{(1)},\ldots,\bx_i^{(v)}, \by_i) \}_{i=1}^n$ be a dataset of $v$ views, where the $i$th data points $\bx_i^{(s)} \in \bbR^{d_s}$ of all views ($1\le s\le v$) are assumed to share the same class labels in $\by_i$
of $c$ labels.

The labels can have different interpretations, dependent of the underlying learning task.
For multi-output regression, $\by_i \in \bbR^{c}$, and it reduces to a scalar  for the classical regression for which $c=1$. For multi-label classification, $\by_i \in \{0, 1\}^{c}$ with an understanding that the $i$th data points of all views
have the class label $r$ if  $(\by_i)_r = 1$ and otherwise $0$, where $(\by_i)_r$ is the $r$th entry of $\by_i$.
If $\bone_c^{\T} \by_i   = 1$, then multi-label classification becomes a problem of $c$-class classification
since one and only one class label is assigned to each instance of data points of all views.
In particular, if $c=2$ and $\bone_c^{\T} \by_i  = 1$, then it is just the binary classification.

For the purpose of feature extraction learning,
objective fulfilling linear transformations are sought to extract the latent representation for each view.
Let $P_s \in \bbR^{d_s \times k}$ be the projection matrix for view $s$ to transform $\bx_i^{(s)}$ from $\bbR^{d_s}$
to $\bz_i^{(s)} = P_s^{\T} \bx_i^{(s)}$ in the common space $\bbR^{k}$. Represent the $n$ data points of view  $s$ by
$X_s = [\bx_1^{(s)},\ldots,\bx_n^{(s)}] \in \bbR^{d_s \times n}$ and its latent representation by
$Z_s = [\bz_1^{(s)} ,\ldots,\bz_n^{(s)} ] = P_s^{\T} X_s \in \bbR^{k \times n}$.
Accordingly, we denote the centered matrix and the sample mean of view $s$, and the label matrix by
\begin{align}\label{eq:c_s}
\widehat{X}_s = X_s H_s, \,\,
\bm_s = \frac{1}{n} X_s \bone_n,\,\,
Y = [\by_1,\ldots,\by_n],
\end{align}
respectively, where $H_n = I_n - \frac{1}{n} \bone_n \bone_n^{\T}$.

The sample cross-covariance between view $s$ and view $t$ is given by
\begin{align}
C_{s,t} = \frac{1}{n} X_s H_n X_t^{\T}. \label{eq:c_st}
\end{align}
In particular, $C_{s,s}$ is the covariance of view $s$.
It is not so hard to  verify that 
\begin{align*}
\widehat{X}_s &= X_s - \bm_s \bone^{\T} = [\bx_1^{(s)} - \bm_s, \ldots, \bx_n^{(s)} - \bm_s],\\
C_{s,t} &= \frac{1}{n} \widehat{X}_s \widehat{X}_t^{\T} = \frac{1}{n} \sum_{i=1}^n (\bx_i^{(s)} - \bm_s) (\bx_i^{(s)} - \bm_s)^{\T},
\end{align*}
upon noticing $H_n^2 = H_n$. 

For the $c$-class classification, i.e., $Y \in \{0,1\}^{c\times n}$ and $\bone_c^{\T} \by_i = 1$,
we have the following properties:
\begin{align}
Y^{\T} \bone_c = \bone_n,  \Sigma =YY^{\T} = \diag(n_1,\ldots,n_c),
\end{align}
where $n_r = \sum_{i=1}^n (\by_i)_r$ is the number of data points in class $r$.
Denote by $\bu_r^{\T}$ the $r$th row of $Y$, and define
\begin{align}
\bm_s^r := \frac{1}{n_r}X_s \bu_r, \forall r,
\end{align}
the mean of all data points in view $s$ having class label $r$. So, we have $\sum_{r=1}^c \bu_r = Y^{\T} \bone_c = \bone_n$.
The between-class scatter matrix $S_b^s$ can be written as
\begin{align}
S_b^s =  X_s (Q - \frac{1}{n} \bone_n \bone_n^{\T}) X_s^{\T}, \label{eq:s_b}
\end{align}
where
$$
Q = Y^{\T} \Sigma^{-1} Y = \sum_{r=1}^c \frac{1}{n_r} \bu_r \bu_r^{\T}.
$$
To see (\ref{eq:s_b}), we note, by definition, that
\begin{align*}
S_b^s &= \sum_{r=1}^c n_r \bm_s^r (\bm_s^r)^{\T} \!-\! n \bm_s \bm_s^{\T} \\
&=\sum_{r=1}^c n_r  ( \bm_s^r  \!-\! \bm_s) ( \bm_s^r  \!-\! \bm_s)^{\T},
\end{align*}
where we have used $$\sum_{r=1}^c n_r \bm_s^r = \sum_{r=1}^c X_s \bu_r = X_s \bone_n = n \bm_s$$ and $\sum_{r=1}^c n_r = n$.
Since $S_w^s  = n C_{s,s} - S_b^s$, the within-class scatter matrix takes the form
\begin{align}
S_w^s = X_s (I - Q) X_s^{\T}. \label{eq:s_w}
\end{align}

A graph Laplacian is  a matrix representation of a graph and popularly used to approximate the manifold structure of data via locality information encoded by the edge weights of the undirected  graph. Denote $G_s=(V_s, E_s, W_s)$ built for $X_s$, where $V_s=\{1, \ldots, n\}$ is the graph nodes and $W_s=[w^{(s)}_{i,j}] \in \bbR^{n\times n}$ is symmetric with its entry $w^{(s)}_{i,j}$ being the weight of edge $(i,j) \in E_s$.
The graph Laplacian matrix of $G_s$ is defined as $L_s = D_s - W_s$, where $D_s = \diag(W_s \bone_n)$.
The manifold approximation is captured by
\begin{align}
\frac{1}{2} \sum_{i=1}^n \sum_{j=1}^n w_{i,j}^{(s)} \| \bz_i^{(s)} - \bz_j^{(s)} \|^2 = \tr(P_s^{\T} X_s L_s X_s^{\T} P_s). \label{eq:graph-reg}
\end{align}
By minimizing (\ref{eq:graph-reg}) with respect to $P_s$,
the optimal projection matrix $P_s$
satisfies the following criterion: if $\bx_i^{(s)}$ is close to $\bx_j^{(s)}$, i.e.,
the similarity $w_{i,j}^{(s)}$ is large, the distance between two corresponding
projected points, $\bz_i^{(s)}$ and $\bz_j^{(s)}$, is also small in the projected space.

We point out that $S_b^s$ and $S_w^s$ in (\ref{eq:s_b}) and (\ref{eq:s_w}) can be expressed in
terms of some graph Laplacians as
\begin{alignat}{2}
S_b^s &= - X_s L_s X_s^{\T}, &\quad\quad W_s &=  Q - \frac{1}{n} \bone_n \bone_n^{\T} , \tag{\ref{eq:s_b}$'$} \\
S_w^s &= X_s L_s X_s^{\T}, &\quad\quad W_s &= Q. \tag{\ref{eq:s_w}$'$}
\end{alignat}
In fact, $S_w^s$ of (\ref{eq:s_w}$'$) was used in the Fisher score for supervised feature selection \cite{he2006laplacian}, and both $S_b^s$ and $S_w^s$ here were used in \cite{sugiyama2007dimensionality} with modified weight matrices to incorporate local information.

\subsection{Single-view Feature Extraction}
A series of single-view (i.e., $v=1$) feature extraction methods  that learn a linear transformation matrix have been proposed. Principal component analysis (PCA) seeks the directions of the input space so that the variance of the projected data is maximized. The joint optimization to obtain transofrmation matrix $P_1$ is formulated as
\begin{align}
\max_{P_1 \in \bbR^{d_1 \times k}} \tr(P_1^{\T} C_{1,1} P_1) : \st~ P_1^{\T} P_1 = I_k. \label{op:pca}
\end{align}
PCA is an unsupervised method since it does not incorporate any output data, and so PCA projections may not be
consistent with output data.

For multi-class classification, linear discriminant analysis (LDA) incorporates output labels to search a projection matrix so that the within-class scatter is minimized while the between-class scatter is maximized. The commonly used LDA formulation is
\begin{align}
\max_{P_1 \in \bbR^{d_1 \times k}} \tr(P_1^{\T} S_b^1 P_1) : \st~ P_1^{\T} N P_1 = I_k. \label{op:lda-rt}
\end{align}
It is popular, in large part because it admits an analytic solution as a generalized eigenvalue problem, where $N$ is either $S_w^1$ or
$C_{1,1}$.
Another approach is the trace ratio formulation for the same purpose
\begin{align}
\max_{P_1 \in \bbR^{d_1 \times k}} \frac{\tr(P_1^{\T} S_b^1 P_1)}{\tr(P_1^{\T} M P_1)} : \st~ P_1^{\T} N P_1 = I_k, \label{op:lda-tr}
\end{align}
where
$M$ and $N$
can be one of $C_{1,1}$, $S_w^1$ and $I_n$. Some of the interesting combinations are as follows:
\begin{enumerate}
	\item (\ref{op:lda-tr}) with $M=C_{1,1}$ \cite{wang2007trace} or with $M=S_w^1$ \cite{zhln:2010}  to obtain an orthonormal projection matrix by letting $N=I_{d_1}$,
	
	\item (\ref{op:lda-tr}) with $M=S_w^1$ and $N=C_{1,1}$  to learn uncorrelated directions \cite{nie2009semi,zhan:2011},
	
	\item (\ref{op:lda-tr}) with  $M=C_{1,1}$ and $N=S_w^1$, equivalent to (\ref{op:lda-rt}) since $n C_{1,1} = S_w^1  + S_b^1$,
	\item (\ref{op:lda-tr}) with $M=N\not=I_{d_1}$, reducing to (\ref{op:lda-rt}).
\end{enumerate}
Another special case is with $M=N=I_{d_1}$ to give
\begin{align}
\max_{P_1 \in \bbR^{d_1 \times k}} \tr(P_1^{\T} S_b^1 P_1) : \st P_1^{\T} P_1 = I_k.
\end{align}

For multi-output regression and multi-label classification, partial least squares (PLS) and canonical correlation analysis (CCA) are two common approaches that also learn projection matrices for output data. PLS looks for a projection matrix that maximizes the cross-covariance between the projected input and output:
\begin{subequations}\label{op:pls}
	\begin{align}
	\max_{P_1 \in \bbR^{d_1 \times k}, P_Y \in \bbR^{c \times k}} &\tr(P_1^{\T} X_1 H_n Y^{\T} P_Y) \label{op:pls-1}\\
	\st &~ P_1^{\T} P_1 = P_Y^{\T} P_Y = I_k, \label{op:pls-2}
	\end{align}
\end{subequations}
where $P_Y \in \bbR^{c\times k}$ is the projection matrix for the output.
CCA also maximizes the cross-correlation
\begin{subequations}\label{op:cca}
	\begin{align}
	\max_{P_1 \in \bbR^{d_1 \times k}, P_Y \in \bbR^{c \times k}}  & \tr(P_1^{\T} X_1 H_n Y^{\T} P_Y) \label{op:cca-1}\\
	\st &~ P_1^{\T} C_{1,1} P_1 = P_Y^{\T} Y Y^{\T} P_Y = I_k, \label{op:cca-2}
	\end{align}
\end{subequations}
except that they have different constraints.
Manifold learning can be used for feature extraction in both supervised and unsupervised settings depending on how the graph is constructed. In \cite{yan2006graph}, a general framework called graph embedding \cite{yan2006graph} is formulated as
\begin{subequations}\label{op:ge}
	\begin{align}
	\min_{P_1} & \tr(P_1^{\T} X_1 L_1 X_1^{\T} P_1) \label{op:ge-1}\\
	\st &~ P_1^{\T} X_1 B X_1^{\T} P_1= I_k \textrm{ or } P_1^{\T} P_1 = I_k, \label{op:ge-2}
	\end{align}
\end{subequations}
where $B$ is to be specified.
According to (\ref{eq:s_b}$'$) and (\ref{eq:s_w}$'$), PCA (\ref{op:pca}) and LDA (\ref{op:lda-rt}) are special cases of (\ref{op:ge}), but the trace ratio formulation (\ref{op:lda-tr}) is not. Locality preserving projection (LPP) \cite{he2004locality} is  (\ref{op:ge}) with graph Laplacian matrix $L_1 = D_1 - W_1$, and
$B=\diag(W_1 \bone_n)$, while Laplacian eigenmap (LE) \cite{belkin2002laplacian}  solves LPP directly for $P_1^{\T} X_1$ instead of $P_1$. As stated in \cite{yan2006graph}, locally linear embedding (LLE) \cite{roweis2000nonlinear} and ISOMAP \cite{tenenbaum2000global} are also spacial cases of (\ref{op:ge}).

\subsection{Multi-view Feature Extraction}\label{ssec:MvFE}
As multiple inputs may come from different sources (views), they are most likely heterogeneous and have large discrepancy. The aim of multi-view feature extraction is to exploit consensual, complementary, and overlapping information among different views.

PLS \eqref{op:pls} and CCA \eqref{op:cca} can be directly applied to two-view data ($v=2$) simply by replacing
$Y$ and $P_Y$ in (\ref{op:pls}) or (\ref{op:cca}) with $X_2$ and $P_2$ of view 2, respectively. For $v > 2$, the multi-set CCA (MCCA) \cite{nielsen2002multiset}
\begin{subequations} \label{op:mcca}
	\begin{align}
	\max_{\{P_s \in \bbR^{d_s \times k}\}}  \sum_{s=1}^v \sum_{t=1}^v \tr(P_s^{\T} C_{s,t} P_t) \\ \st~ \sum_{s=1}^v P_s^{\T} C_{s,s} P_s = I_k, 
	\end{align}
\end{subequations}
is the most popularly used, chiefly due to its analytic solution via the generalized eigen-decomposition that
has been well studied \cite{bddrv:2000,govl:2013}. Orthogonal multiset CCA (OMCCA)
\begin{subequations}\label{op:omcca}
	\begin{align}
	\max_{\{P_s \in \bbR^{d_s \times k}\}} & \sum_{s=1}^v \sum_{t=1}^v \frac{ \tr(P_s^{\T} C_{s,t} P_t) }
	{\sqrt{\tr(P_s^{\T} C_{s,s} P_s )\vphantom{P_t^{\T}}}
		\sqrt{\tr(P_t^{\T} C_{t,t} P_t )}} \label{op:omcca-1}\\
	\st &~ P_s^{\T} P_s = I_k, \forall s \label{op:omcca-2}
	\end{align}
\end{subequations}
is proposed in \cite{Zhang2020Self}. Its special case $v=2$ is the orthogonal CCA (OCCA) \cite{wang2020orthogonal,shsy:2013,cugh:2015}. In \cite{shen2015orthogonal},
a variant of \eqref{op:omcca} was studied.
The key in (\ref{op:mcca}) and (\ref{op:omcca}) is the use of pairwise cross-covariance matrices $\{C_{s,t}\}$ to capture the consensus among the $v$ views.

Recently, PLS is extended for $v>2$ in \cite{wang2020scalable}, too, where the orthogonality constraints $P_s^{\T} P_s = I_k$ for all $s$ are imposed.

For supervised learning, the output label $Y$ can be naturally considered as one input view \cite{sun2010canonical}. However, the special structure of label information is neglected.
To compensate that negligence and to take full advantage of label data, sophisticated multi-view feature extraction methods have been proposed. In \cite{sharma2012generalized}, generalized multi-view analysis (GMA) is formulated, by integrating LDA (or some variants of it) and CCA, as
\begin{subequations}\label{op:gma}
	\begin{align}
	\max_{\{P_s\}} & \sum_{s=1}^v \tr(P_s^{\T} S_b^s P_s) +  \sum_{s=1}^v \sum_{t=1, t \not=s}^v \alpha_{s,t} \tr(P_s^{\T} C_{s,t} P_t) \label{op:gma-1}\\
	\st &~ P_s^{\T} S_w^s P_s = I_k, \forall s, \label{con:gma}
	\end{align}
\end{subequations}
where $\alpha_{s,t}$ is the weight for cross-covariance between view $s$ and view $t$.
Unfortunately, this is a difficult optimization problem whose KKT condition leads to a multi-parameter eigenvalue problem like \eqref{eq:MDA-eig2} later
for which there is no efficient numerical method for its solution. For that reason,
authors in \cite{sharma2012generalized} proposed to solve, instead, a relaxed problem: the same objective but a  constraint different from (\ref{con:gma}):
\begin{align}
\sum_{s=1}^v \gamma_s P_s^{\T} S_w^s P_s = I_k, \label{eq:gma-con}
\end{align}
resulting in a generalized eigenvalue problem \cite{govl:2013}, where $\gamma_s$ are parameters to balance $v$ independent constraints. $S_b^s$ and $S_w^s$ can be the ones in (\ref{eq:s_b}$'$) and (\ref{eq:s_w}$'$) for the classical LDA, or those in \cite{sugiyama2007dimensionality,yan2006graph}. Multi-view uncorrelated linear discriminant analysis (MULDA) \cite{sun2015multiview} was proposed to replace (\ref{con:gma}) with the uncorrelated constraints
\begin{align} \label{con:mulda}
\sum_{s=1}^v \gamma_s P_s^{\T} C_{s,s} P_s = I_k.
\end{align}
Multi-view modular discriminant analysis (MvMDA) \cite{cao2017generalized} aims to maximize the distances between different class centers across different views and minimize the within-class scatter
\begin{subequations}\label{op:mvmda}
	\begin{align}
	\max_{\{P_s\}}  & \sum_{s=1}^v \sum_{t=1}^v \tr(P_s^{\T} X_s A X_t^{\T} P_t) \\
	\st & \sum_{s=1}^v P_s^{\T} S_w^s P_s \!=\! I_k, \label{op:mvmda-2}
	\end{align}
\end{subequations}
where $A = Y^{\T} \Sigma^{-1} H_c \Sigma^{-1} Y$.

It is worth noting that imposing orthogonality constraints  has attracted much attention in multi-view feature extraction
in unsupervised learning, but it is seldom explored in supervised learning.
In addition, it has been widely studied in single-view methods in both unsupervised and supervised learning.

\section{Orthogonal Multi-view Analysis} \label{sec:omva}
In this section, we propose a novel unified framework for multi-view discriminant analysis in order to learn orthogonal projections onto a latent common space.

\subsection{Motivation}
An orthogonal projection is able to preserve the pairwise distance if the vectors to be projected live
in the range of the projection. Specifically, if $\bx_i^{(s)}\in\cR(P_s)$ for all $i$ and $P_s^{\T} P_s=I_k$, then we have
$\bx_i^{(s)}=P_s\wtd\bz_i^{(s)}$ for some $\wtd\bz_i^{(s)}$ and
$$
P_s \bz_i^{(s)}=P_s(P_s^{\T} \bx_i^{(s)})=P_s\underbrace{P_s^{\T}P_s}\wtd\bz_i^{(s)} =
P_s\wtd\bz_i^{(s)}=\bx_i^{(s)}.
$$
Now, the pairwise Euclidean distance between $\bx_i^{(s)}$ and $\bx_j^{(s)}$
\begin{align}
\!\| \bx_i^{(s)} - \bx_j^{(s)} \|^2 = \| P_s (\bz_i^{(s)}  - \bz_j^{(s)})  \|^2 = \| \bz_i^{(s)} - \bz_j^{(s)} \|^2
\end{align}
is preserved in the projected space. Distance preservation as an important learning criterion has been successfully used in single-view dimensionality reduction with kernel representation \cite{weinberger2004learning} and Bayesian estimation \cite{wang2017latent}.

Orthogonal projection has been explored in LDA (\ref{op:lda-tr}) with $N=I_{d_1}$ in \cite{wang2007trace,zhln:2010} for single-view feature extraction, and in CCA with two views \cite{wang2020orthogonal,shsy:2013,cugh:2015} and  MCCA with more than two views \cite{Zhang2020Self,shen2015orthogonal} for multi-view feature extraction. However, imposing
orthogonality constraints has not yet been well studied for supervised multi-view subspace learning.

\subsection{A Unified Framework}
We propose a novel unified orthogonal multi-view subspace learning (OMvSL) framework in the trace ratio formulation given by
\begin{subequations}\label{op:ouf}
	\begin{align}
	\max_{\{P_s\}} & \sum_{s=1}^v \sum_{t=1}^v \frac{ \tr(P_s^{\T} \Phi_{s,t} P_t) }
	{\sqrt{\tr(P_s^{\T} \Psi_{s,s} P_s )\vphantom{P_t^{\T}}} \sqrt{\tr(P_t^{\T} \Psi_{t,t} P_t )}}
	\label{op:ouf-1}\\
	\st &~ P_s^{\T} P_s = I_k, \forall s, \label{op:ouf-2}
	\end{align}
\end{subequations}
where $\Psi_{s,s}$ for $s=1,\ldots, v$ are positive semi-definite matrices.
As stated in \cite{wang2007trace}, the trace ratio formation is an essential formulation for  general dimensionality reduction and may lead to solutions that are superior to the ones
from the ratio trace formulation.

The proposed OMvSL \eqref{op:ouf} encompasses OLDA and OMCCA as special cases:
\begin{enumerate}
	\item For $v=1$, (\ref{op:ouf}) with $\Psi_{1,1} = S_b^1$ and $\Psi_{1,1} = M$ reduces to OLDA (\ref{op:lda-tr}).
	\item For $v\geq 2$, (\ref{op:ouf}) with $\Phi_{s,t} = C_{s,t}$ and $\Psi_{s,s} = C_{s,s}$ becomes OMCCA (\ref{op:omcca}).
\end{enumerate}
OMvSL \eqref{op:ouf} can be used to inspire various models in the form of trace ratio formulations.
We  shall present various novel models instantiated from OMvSL (\ref{op:ouf}) for multi-view discriminant analysis in subsection \ref{sec:mvda} and multi-label classification in subsection~\ref{sec:mlabel}.

OMvSL is a versatile framework, but it presents a difficult optimization problem to solve.
Generic optimization techniques \cite{abms:2008,nowr:2006,weyi:2013} can always be applied, but they ignore the special form in the objective, are usually not so efficient as customized algorithms, and, worst of all, are not practically feasible even for   datasets of modest scale.
In Section~\ref{sec:alg4OMDA}, we will present a successive approximation algorithm  that approximately solves OMvSL efficiently.

It is tempting to modify OMvSL (\ref{op:ouf}) by adding
\begin{align}
P_s^{\T} \Psi_{s,s} P_s = I_k, \forall s, \label{eq:ind}
\end{align}
to eliminate the denominators in the objective
in hope for a simpler problem to solve. But \eqref{eq:ind} and \eqref{op:ouf-2} may conflict.
To see that, we note that $P_s^{\T} \Psi_{s,s} P_s\succeq \lambda_{\min} P_s^{\T} P_s$ where
$\lambda_{\min}$ is the smallest eigenvalue of $\Psi_{s,s}$, and so if $\lambda_{\min}>1$, then there is no way to
satisfy both \eqref{eq:ind} and \eqref{op:ouf-2} at the same time.
On the other hand, \eqref{op:ouf-1} with \eqref{eq:ind} but not \eqref{op:ouf-2}
bears similarity to existing models of the ratio trace formulation in subsection~\ref{ssec:MvFE}.

\subsection{Novel Multi-view Discriminant Analysis Models} \label{sec:mvda}
Three orthogonal multi-view discriminant analysis models are proposed, inspired by existing  models similar to (\ref{op:ouf}) for multi-class classification where $\by_i \in \{0,1\}^{c}$ and $\by_i^{\T} \bone_c=1$ \cite{sharma2012generalized,cao2017generalized,sun2015multiview}.
Each new model is  intrinsically different from its corresponding existing model due to the trace ratio formulation (\ref{op:ouf-1}) and orthogonality constraints (\ref{op:ouf-2}).

\noindent {\bf Orthogonal GMA.}
The proposed orthogonal variant of GMA (\ref{op:gma}), called {\em Orthogonal GMA} (OGMA), is (\ref{op:ouf}) with
\begin{subequations}\label{eq:phi_gma}
	\begin{align}
	\Phi_{s,t} &= \left\{
	\begin{array}{ll}
	S_b^s, & s=t, \\
	\alpha_{s,t} C_{s,t}, & s \ne t,
	\end{array}
	\right. \label{eq:phi_gma-1}\\
	\Psi_{s,s} &= S_w^s . \label{eq:phi_gma-2}
	\end{align}
\end{subequations}

\noindent {\bf Orthogonal MLDA.}
The  proposed orthogonal variant of MLDA (\ref{op:gma-1}) with (\ref{con:mulda}), called {\em Orthogonal MLDA} (OMLDA),
is (\ref{op:ouf}) with  
(\ref{eq:phi_gma-1}) and
\begin{align}
\Psi_{s,s} = C_{s,s}. \label{eq:mlda-osi}
\end{align}

\noindent {\bf Orthogonal MvMDA.}
The proposed orthogonal variant of  MvMDA (\ref{op:mvmda}), called {\em Orthogonal MvMDA} (OMvMDA),
is (\ref{op:ouf}) with
\begin{align}
\Phi_{s,t}  = A, ~ \Psi_{s,s} = S_w^s. \label{eq:mvmda}
\end{align}

\subsection{Novel Multi-view Multi-label Classification Models} \label{sec:mlabel}
In multi-view multi-label classification, the output $\by_i \in \{0,1\}^c$ with $c$ labels and $\{ \bx_i^{(1)}, \ldots, \bx_i^{(v)}, \by_i \}_{i=1}^n$ is the paired data. Under the proposed framework (\ref{op:ouf}), we can come up the following two strategies to incorporate output data for multi-view multi-label classification:

\noindent{\bf Orthogonal Multi-view Multi-label CCA (OM$^2$CCA).} This approach is proposed to take the output
labels in $Y=[\by_1,\ldots, \by_n] \in \{0,1\}^{c \times n}$ as the $(v+1)$st view $X_{v+1} :=Y$ in OMCCA  \cite{sun2010canonical}. Together with $v$ input views, there are $v+1$ views. OMCCA is employed to learn projection matrices $\{P_s\}$ and $P_{v+1}:=P_Y$ in a latent common space. This idea has been explored for $v=1$ in \cite{sun2010canonical,wang2020orthogonal,Zhang2020Self}. OMCCA is instantiated from (\ref{op:ouf}) with
\begin{subequations}\label{eq:omcca}
	\begin{align}
	\Phi_{s,t} &= \left\{
	\begin{array}{ll}
	0, & s=t, \\
	C_{s,t}, & s \ne t,
	\end{array}
	\right. \label{eq:ommcca-1}\\
	\Psi_{s,s} &= C_{s,s}, \label{eq:omcca-2}
	\end{align}
\end{subequations}
for $s, t=1,\ldots,v+1$, where $C_{s,v+1} = X_s H Y = C_{v+1,s}^{\T}$.

\noindent{\bf Orthogonal Hilbert-Schmidt Independence Criterion (OHSIC).}  This approach is proposed to take the HSIC criterion \cite{gretton2005measuring} for learning embedding of each input view. The estimator of HSIC is defined as
\begin{align}
\textrm{HSIC}(Z_s, Y) = \frac{1}{(n-1)^2} \tr(Z_s^{\T} Z_s H_n Y^{\T} Y H_n),
\end{align}
where $Z_s = P_s^{\T} X_s$ and $Z_s^{\T} Z_s$ is the linear kernel of the projected data of view $s$. To achieve the best alignment between $Z_s$ and $Y$, the maximization of HSIC with respect to $P_s$ is expected. The proposed HSIC method is instantiated from (\ref{op:ouf})  with
\begin{subequations}\label{eq:hsic}
	\begin{align}
	\Phi_{s,t} &= \left\{
	\begin{array}{ll}
	X_sH_n Y^{\T} Y H_n X_s^{\T} , & s=t, \\
	\alpha_{s,t} C_{s,t}, & s \ne t,
	\end{array}
	\right. \label{eq:hsic-1}\\
	\Psi_{s,s} &= C_{s,s} , \label{eq:hsic-2}
	\end{align}
\end{subequations}
for $s, t=1,\ldots, v$.
Different from (\ref{eq:omcca}), this approach does not  learn $P_Y$.

\section{An eigenvalue algorithm}\label{sec:geig}

Currently there is no numerically efficient method to solve OMvSL \eqref{op:ouf}, especially for high-dimensional datasets.
In preparing for presenting a successive approximation method  in the next section,
in what follows we will outline a Krylov subspace method that is suitable for computing the top eigenpair for the generalized eigenvalue problem.
To simplify notation, we will explain the method generically for
\begin{equation}\label{eq:eig-generic}
A\bx=\lambda B\bx
\quad\mbox{with}\quad
\bx\in\cR(B),
\end{equation}
where $A,\, B\in\bbR^{d\times d}$ are symmetric,  $\cR(A)\subseteq\cR(B)$, $B\succeq 0$. Suppose that matrix-vector products, $A\bx$ and $B\bx$ for
any given $\bx$, are the only operations that can be done numerically.

The Krylov subspace method will serve as the workhorse of our
successive approximation algorithm  that approximately solves OMvSL  \eqref{op:ouf}. It is worth noting that $B$ may be singular
and will be singular in our applications. A common past practice in data science is simply to perturb $B$ to
$B+\epsilon I_d$ for some tiny $\epsilon>0$ as a regularization and solve
$A\bx=\lambda (B+\epsilon I_d)\bx$ instead. While this successfully gets rid of the singularity issue, it may create a more
serious one in that the eventually computed top eigenvector likely falls into the null spaces of $A$ and $B$ and is thus useless for the underlying application.

The method is the so-called Locally Optimal Block Preconditioned Extended Conjugate Gradient method (LOBPECG)
\cite[Algorithm 2.3]{li:2015} which combines LOBPCG of Knyazev \cite{knya:2001} and
the {\em inverse free Krylov subspace method\/} of Golub and Ye~\cite{goye:2002}. For our current application, we will simply
use the version without preconditioning and blocking. Algorithm~\ref{alg:eLOCG} outlines an adaption
of \cite[Algorithm 2.3]{li:2015} for \eqref{eq:eig-generic}.

\begin{algorithm}
	\caption{Locally Optimal Extended Conjugate Gradient method (LOECG)}\label{alg:eLOCG}
	\begin{algorithmic}[1]
		\REQUIRE eigenvalue problem \eqref{eq:eig-generic}, $\nkry$, tolerance $\tol$;
		\ENSURE  top eigenpair $(\lambda,\bx)$.
		\STATE  pick a random $\bx_1\in\bbR^d$;
		\STATE  $\bx_1=B\bx_1$, $\bx_1=\bx_1/\|\bx_1\|_2$, $\rho=\bx_1^{\T}A\bx_1/\bx_1^{\T}B\bx_1$;
		\STATE  $\br=A\bx_1-\rho B\bx_1$, $\algres=\|\br\|_2/(\|A\|_2+|\rho|\|B\|_2)$;
		\STATE  $\bx_0=0$;
		\WHILE{$\algres\ge\tol$}
		\STATE compute an orthonormal basis matrix $Z$ of the Krylov subspace
		\begin{equation}\label{eq:eLOCG:subspace}
		\!\!\!\cR(Z)\!=\!\cR([\bx_1, (A-\rho B)\bx_1,\ldots,(A-\rho B)^{\nkry}\bx_1]);\!\!
		\end{equation}
		\STATE $\bp=\bx_0-Z(Z^{\T}\bx_0)$, $W=[Z,\bp/\|\bp\|_2]$;
		\STATE compute the top eigenpair $(\rho,\bz)$ of $W^{\T}AW-\lambda W^{\T}BW$, where $\|\bz\|_2=1$;
		\STATE $\bx_0=\bx_1$;
		\STATE $\bx_1=W\bz$, $\br=A\bx_1-\rho B\bx_1$, $\algres=\|\br\|_2/(\|A\|_2+|\rho|\|B\|_2)$;
		\ENDWHILE
		\RETURN $(\rho,\bx_1)$.
	\end{algorithmic}
\end{algorithm}

A few comments regarding this algorithm and its efficient implementation are in order:
\begin{enumerate}
	\item There is no need to use $\|A\|_2$ and $\|B\|_2$ exactly. Some very rough estimates are just good enough so long as
	the estimates have the same magnitudes, respectively.
	\item At line 2, it is to make sure $\bx_1\in\cR(B)$.
	\item There are two parameters to choose: the order $\nkry$ of the Krylov space \eqref{eq:eLOCG:subspace} and the stopping tolerance $\tol$. There is no easy way to determine what the optimal $\nkry$ is. In general, the larger $\nkry$ is, the faster the convergence, but then more work in generating the orthonormal basis matrix $Z$.
	Usually $\nkry=10$ is good. For applications that required accuracy is not too stringent, $\tol=10^{-6}$ is often more than adequate.
	\item The orthonormal basis matrix $Z$ can be efficiently computed by the symmetric Lanczos process \cite{demm:1997}. For better
	numerical stability in making sure $Z^{\T}Z=I$ within the working precision, re-orthogonalization may be necessary.
	\item At line 7, some guard step must be taken. For example, in the first iteration $\bx_0=0$ and so $\bp=0$. We should just let $W=Z$.
	In the subsequent iterations, we will have to test whether $\bx_0$ is in or nearly in $\cR(Z)$. For that purpose, we need another tolerance, e.g., if $\|\bp\|_2\le 10^{-12}$, then we will regard already $\bx_0\in\cR(Z)$ and set $W=Z$; otherwise,
	re-orthogonalize $\bp$ against $Z$: $\bp=\bp-Z(Z^{\T}\bp)$ to make sure $W^{\T}W=I$ within the working precision.
	\item At line 8, $AW$ and $BW$, except their last columns, are likely already computed at the time of generating $Z$ at line 6.
	They should be reused here to save work.
	\item The eigenvalue problem for $W^{\T}AW-\lambda W^{\T}BW$ is of very small size $(\nkry+1)\times(\nkry+1)$ at most and
	also $W^{\T}BW\succ 0$ as guaranteed by Lemma~\ref{lm:Rc} below. It can be solved by first computing the Cholesky
	decomposition $W^{\T}BW=R^{\T}R$ and then the full eigen-decomposition of $R^{-\T}(W^{\T}AW)R^{-1}$. Finally, $\bz=R^{-1}\bw$, where
	$\bw$ is the top eigenvector of $R^{-\T}(W^{\T}AW)R^{-1}$.
\end{enumerate}

\begin{lemma}\label{lm:Rc}
	In Algorithm~\ref{alg:eLOCG}, $\cR(W)\subseteq\cR(B)$ and thus $W^{\T}BW\succ 0$.
\end{lemma}

\begin{proof}
	Initially, after line 2, $\bx_1\in\cR(B)$.
	Therefore at \eqref{eq:eLOCG:subspace}, $\cR(Z)\subseteq\cR(B)$ because $\cR(A)\subseteq\cR(B)$. In the first iteration of the {\bf while}-loop,
	$\bx_0=0$ and $W=Z$ and so $\cR(W)\subseteq\cR(B)$, $\bx_0,\,\bx_1\in\cR(B)$. Inductively,
	each time at the beginning of executing the {\bf while}-loop, we have $\bx_0,\,\bx_1\in\cR(B)$. So we will have
	at line 7, $\bp\in\cR(B)$ and $\cR(Z)\subseteq\cR(B)$, implying $\cR(W)\subseteq\cR(B)$. Consequently,
	at the conclusion of executing the {\bf while}-loop, we still have $\bx_0,\,\bx_1\in\cR(B)$.
	
	Since $B\succeq 0$ and  $\cR(W)\subseteq\cR(B)$, $W^{\T}BW$ must be positive definite.
\end{proof}

\section{Algorithm for OMvSL}\label{sec:alg4OMDA}

For ease of presentation, we rewrite OMvSL \eqref{op:ouf} as
\begin{equation}\label{eq:OMDA}
\max_{\{P_s\}}  g(\{P_s\}): \,
\st ~ P_s^{\T}P_s=I_k,\,\,\cR(P_s)\subseteq\cR(\Psi_{s,s})\, \forall s,
\end{equation}
where
$$
g(\{P_s\}):=\sum_{s=1}^v \sum_{t=1}^v \frac{ \tr(P_s^{\T} \Phi_{s,t} P_t) }
{\sqrt{\tr(P_s^{\T} \Psi_{s,s} P_s )\vphantom{P_t^{\T}}} \sqrt{\tr(P_t^{\T} \Psi_{t,t} P_t )}}.
$$
For $k=1$, all $P_s$ are column vectors. By convention that we  use lowercase letters
for vectors, we will replace them by $\bp_s$ instead. Since $g(\{\bp_s\})$ is homogeneous in each $\bp_s$, i.e.,
$g(\{\bp_s/\alpha_s\})\equiv g(\{\bp_s\})$ for any scalar $\alpha_s>0$, the constraints $\bp_s^{\T}\bp_s=1$ is inconsequential. In fact, \eqref{eq:OMDA} is  equivalent to
\begin{equation}\label{eq:MDA-vec2}
\max_{\{\bp_s\in\bbR^{n_s}\}} f(\{\bp_s\})\,:\,\mbox{s.t.}\,\, \bp_s^{\T}\Psi_{s,s}\bp_s=1,\,\bp_s\in\cR(\Psi_{s,s})
\,\forall s,
\end{equation}
where $f(\{\bp_s\})$ is given by
\begin{equation}\label{eq:f(ps)}
f(\{\bp_s\}):=\sum_{s=1}^v \sum_{t=1}^v   \bp_s^{\T} \Phi_{s,t}\bp_t.
\end{equation}
The KKT condition of \eqref{eq:MDA-vec2} gives rise to a multi-parameter eigenvalue problem:
\begin{subequations}\label{eq:MDA-eig2}
	\begin{equation}\label{eq:MDA-eig2-1}
	\scrA\bp=\scrB\Lambda\bp,\,\,\bp\in\cR(\scrB),
	\end{equation}
	where
	\begin{equation}\label{eq:MDA-eig2-2}
	\scrA=\begin{bmatrix}
	\Phi_{11} & \Phi_{12} & \cdots & \Phi_{1v} \\
	\Phi_{21} & \Phi_{22} & \cdots & \Phi_{2v} \\
	\vdots & \vdots & \ddots & \vdots \\
	\Phi_{v 1} & \Phi_{v 2} & \cdots & \Phi_{vv}
	\end{bmatrix},\,
	\scrB=\begin{bmatrix}
	\Psi_{11} &  &  &  \\
	& \Psi_{22} &  &  \\
	&  & \ddots &  \\
	&  &  & \Psi_{vv}
	\end{bmatrix},
	\end{equation}
	\begin{equation}\label{eq:MDA-eig2-3}
	\Lambda
	=\begin{bmatrix}
	\lambda_1I_{n_1} &  &  &  \\
	& \lambda_2I_{n_2} &  &  \\
	&  & \ddots &  \\
	&  &  & \lambda_{v}I_{n_{v}}
	\end{bmatrix},~~\bp=\left[\begin{array}{c}\bp_1 \\\vdots \\\bp_v\end{array}\right].
	\end{equation}
\end{subequations}
This is also a long standing problem in statistics, and there is no existing numerical technique that is readily available to solve it with guarantee. Existing methods  include variations of the power method for matrix eigenvalues 
\cite{chwa:1993},
which are simple to use but often slowly convergent, and adaptations of common optimization techniques onto  Riemannian manifolds to solve \eqref{eq:MDA-vec2}
\cite{zhan:2010,zhan:2012},
which often converge faster but  use the gradient or even Hessian of $f$ and, as a result, are not particularly well suited for large scale problems. None of those methods guarantee to deliver the global optimum of \eqref{eq:MDA-vec2}. 

In many real-world applications, an approximate solution is just as good as a very accurate solution. A relaxed problem to
\eqref{eq:MDA-vec2} is
\begin{equation}\label{eq:MDA-vec3}
\max_{\{\bq_s\}} f(\{\bq_s\})\,:\,\,\mbox{s.t.}\,\, \sum_{s=1}^v\bq_s^{\T}\Psi_{s,s}\bq_s=1,\,\bq_s\in\cR(\Psi_{s,s}).
\end{equation}
The KKT condition for \eqref{eq:MDA-vec3} is
\begin{equation}\label{eq:MCCA-eig-1}
\scrA \bq=\lambda\scrB \bq,\,\,\bq\in\cR(\scrB)
\end{equation}
which is a generalized eigenvalue problem that has been well studied, where $\scrA$ and $\scrB$ are as given by \eqref{eq:MDA-eig2-2}.
Often
$$
\cR(\scrA)\subseteq\cR(\scrB)
$$
which we will assume in this paper and the top eigenvector $\bq$  is the maximizer of \eqref{eq:MDA-vec3}.
Even though $\scrB$ is positive semi-definite, it is possible that $\scrB$ is singular.
In the previous section, we discussed a common practice and its fatal shortcoming for
the underlying data science application.
Algorithm~\ref{alg:eLOCG} in section~\ref{sec:geig} can be applied
to solve \eqref{eq:MCCA-eig-1} for its top eigenpair in such a way that singular $\scrB$ does not matter, without
regularization.

We propose to
construct an approximation solution for \eqref{eq:MDA-vec2}, and thereby for \eqref{eq:OMDA} with $k=1$, from
the solution to \eqref{eq:MDA-vec3} for $k=1$ as follows.
Let $(\lambda_1,\bq^{\opt}=[\bq_s^{\opt}])$ with $\bq_s^{\opt}\in\bbR^{d_s}$ be the top eigenpair of the eigenvalue problem
\eqref{eq:MCCA-eig-1}. An approximate solution is then constructed by
\begin{equation}\label{eq:MCCA1->MCCA2}
\gamma_s=\|\bq_s^{\opt}\|_2~,\,\,
\bp_s^{\opt}= \bq_s^{\opt}/\gamma_s,\,\,\forall s.
\end{equation}
This solves \eqref{eq:OMDA} with $k=1$ approximately, or finds an approximation to the first columns of optimal $P_s$ of
\eqref{eq:OMDA}.
Suppose that
approximations to the first $\ell$ columns, say  $\bp_s^{(j)}\in\bbR^{d_s}$ for $1\le j\le \ell$, of nearly optimal $P_s$ of \eqref{eq:OMDA}
are obtained and $\ell<k$. Let
\begin{equation}\label{eq:Yt}
P_s^{(\ell)}=\big[\bp_s^{(1)},\bp_s^{(2)},\ldots,\bp_s^{(\ell)}\big]\in\bbR^{d_s\times\ell},\,\,\forall s.
\end{equation}
It is reasonable to assume
\begin{equation}\label{eq:assumeYt}
[P_s^{(\ell)}]^{\T}P_s^{(\ell)}=I_\ell,\,\,\cR(P_s^{(\ell)})\subseteq\cR(\Psi_{s,s}),\,\,\forall s.
\end{equation}
We propose
to find the next columns of nearly optimal $P_s$ for all  $s$ of \eqref{eq:OMDA} by solving
\begin{subequations}\label{eq:OMDA-vec-SS-t}
	\begin{align}
	\max_{\{\bq_s\in\bbR^{n_s}\}} f(\{\bq_s\}):\,\,\mbox{s.t.}&\quad \sum_{s=1}^{v}\bq_s^{\T}\Psi_{s,s}\bq_s=1,
	\,\,\bq_s\in\cR(\Psi_{s,s})\,\forall s,         \label{eq:OMDA-vec-SS-t-1}\\
	&\quad \bq_s^{\T}P_s^{(\ell)}=0\,\forall s, \label{eq:OMDA-vec-SS-t-2}
	\end{align}
\end{subequations}
and then normalize each $\bq_s$ of the optimizer of \eqref{eq:OMDA-vec-SS-t} as in \eqref{eq:MCCA1->MCCA2} to construct the next $\bp_s^{(\ell+1)}$.

\begin{theorem}\label{thm:OMDA-vec-SS-t}
	Given $P_s^{(\ell)}$ as in \eqref{eq:Yt} satisfying \eqref{eq:assumeYt},
	problem \eqref{eq:OMDA-vec-SS-t} is equivalent to
	\begin{equation}\label{eq:OMDA-vec-SS-t'}
	\max_{\{\bq_s\in\bbR^{n_s}\}} f_{\ell}(\{\bq_s\}):\,\,\mbox{s.t.}\, \sum_{s=1}^{v}\bq_s^{\T}\Psi_{s,s}^{(\ell)}\bq_s=1,
	\,\,\bq_s\in\cR(\Psi_{s,s}^{(\ell)})\,\forall s,
	\end{equation}
	where
	\begin{subequations}\label{eq:Pi-s-PhiPsi-ell}
		\begin{align}
		\Pi_s^{(\ell)}&=I_{n_s}-P_s^{(\ell)}\big[P_s^{(\ell)}\big]^{\T}, \label{eq:Pi-s-PhiPsi-ell-1}  \\
		\Phi_{s,t}^{(\ell)}&=\Pi_s^{(\ell)}\Phi_{s,t}\Pi_t^{(\ell)},\,\, \Psi_{s,s}^{(\ell)}=\Pi_s^{(\ell)}\Psi_{s,s}\Pi_s^{(\ell)},
		\label{eq:Pi-s-PhiPsi-ell-2}  \\
		f_{\ell}(\{\bq_s\})&=\sum_{s,t}\bq_s^{\T}\Phi_{s,t}^{(\ell)}\bq_t.     \label{eq:Pi-s-PhiPsi-ell-3}
		\end{align}
	\end{subequations}
\end{theorem}

\begin{proof}
	We will show that the feasible sets for \eqref{eq:OMDA-vec-SS-t} and \eqref{eq:OMDA-vec-SS-t'} are the same and
	$f(\{\bq_s\})=f_\ell(\{\bq_s\})$ for any vector $\{\bq_s\}$ in the feasible set.
	
	Let $\{\bq_s\}$ satisfy  the constraints of \eqref{eq:OMDA-vec-SS-t}.
	Since $\bq_s^{\T}P_s^{(\ell)}=0$, we have $\Pi_s^{(\ell)}\bq_s=\bq_s$. Since
	$\bq_s\in\cR(\Psi_{s,s})=\cR(\Psi_{s,s}^{1/2})$ where $\Psi_{s,s}^{1/2}$ is the unique positive semi-definite square root of
	$\Psi_{s,s}$, we have $\bq_s=\Psi_{s,s}^{1/2}\bw_s$ for some $\bw_s$. Therefore
	\begin{align*}
	\bq_s=\Pi_s^{(\ell)}\bq_s&=\Pi_s^{(\ell)}\Psi_{s,s}^{1/2}\bw_s \\
	&\in\cR(\Pi_s^{(\ell)}\Psi_{s,s}^{1/2}) =\cR(\Pi_s^{(\ell)}\Psi_{s,s}\Pi_s^{(\ell)}), \\
	\bq_s^{\T}\Phi_{s,t}\bq_t&=\big[\Pi_s^{(\ell)}\bq_s\big]^{\T}\Phi_{s,t}\big[\Pi_s^{(\ell)}\bq_t\big]
	=\bq_s^{\T}\Phi_{s,t}^{(\ell)}\bq_t.
	\end{align*}
	Hence $\{\bq_s\}$ satisfies the constraints of \eqref{eq:OMDA-vec-SS-t'} and $f(\{\bq_s\})=f_{\ell}(\{\bq_s\})$.
	On the other hand, let $\{\bq_s\}$ satisfy  the constraints of \eqref{eq:OMDA-vec-SS-t'}.
	Since $\bq_s\in\cR(\Psi_{s,s}^{(\ell)})=\cR(\Pi_s^{(\ell)}\Psi_{s,s}^{1/2})$,
	we have $\bq_s=\Pi_s^{(\ell)}\Psi_{s,s}^{1/2}\bw_s$ for some $\bw_s$ and therefore
	\begin{align*}
	&\bq_s^{\T}P_s^{(\ell)}=\bw_s^{\T}\Psi_{s,s}^{1/2}\Pi_s^{(\ell)}P_s^{(\ell)}=0,\\
	&\bq_s=\Psi_{s,s}^{1/2}\bw_s-P_s^{(\ell)}\big[P_s^{(\ell)}\big]^{\T}\Psi_{s,s}^{1/2}\bw_s\in\cR(\Psi_{s,s}^{1/2})=\cR(\Psi_{s,s}).
	\end{align*}
	That $\bq_s^{\T} P_s^{(\ell)}=0$ implies $\Pi_s^{(\ell)}\bq_s=\bq_s$ for all $s$, and therefore
	$$
	\bq_s^{\T}\Phi_{s,t}\bq_t=\bq_s^{\T}\Pi_s^{(\ell)}\Phi_{s,t}\Pi_t^{(\ell)}\bq_t=\bq_s^{\T}\Phi_{s,t}^{(\ell)}\bq_t.
	$$
	Hence also $\{\bq_s\}$ satisfies  the constraints of \eqref{eq:OMDA-vec-SS-t} and $f(\{\bq_s\})=f_{\ell}(\{\bq_s\})$.
\end{proof}

In view of our previous discussion,  problem \eqref{eq:OMDA-vec-SS-t'} is equivalent to finding the top eigenpair of
\begin{equation}\label{eq:MCCA-eig-t}
\scrA^{(\ell)}\bq=\lambda\scrB^{(\ell)}\bq\quad\mbox{with $\bq\in\cR(\scrB^{(\ell)})$},
\end{equation}
where $\scrA^{(\ell)}$ and $\scrB^{(\ell)}$ take the same form as $\scrA$ and $\scrB$ in
\eqref{eq:MDA-eig2-2}, except with all
$\Phi_{s,t}$ and $\Psi_{s,s}$ replaced by $\Phi_{s,t}^{(\ell)}$ and $\Psi_{s,s}^{(\ell)}$, respectively. Note now that $\scrB^{(\ell)}$ is guaranteed singular for $\ell>1$ because for each $s$,
\begin{align*}
\rank(\Psi_{s,s}^{(\ell)})&=\rank(\Pi_s^{(\ell)}\Psi_{s,s}^{1/2}) \\
&\le\min\{\rank(\Pi_s^{(\ell)}),\,\rank(\Psi_{s,s}^{1/2})\} \\
&\le \rank(\Pi_s^{(\ell)})=n_s-\ell.
\end{align*}
Hence the range constraint $\bq\in\cR(\scrB^{(\ell)})$ is indispensable.
Any straightforward application of existing eigen-computation routine to
$\scrA^{(\ell)}-\lambda\scrB^{(\ell)}$ will likely encounter some numerical issue.
Note that $\bq\in\cR(\scrB^{(\ell)})$ is equivalent to
$\bq_s\in\cR(\Psi_{s,s}^{(\ell)})\,\forall s$ in \eqref{eq:OMDA-vec-SS-t'}.

Algorithm~\ref{alg:RiOMDA} summarizes our range constrained successive approximation  method for solving OMvSL, which calls
Algorithm~\ref{alg:eLOCG} to compute the top eigenvector of
$\scrA^{(\ell)}-\lambda\scrB^{(\ell)}$, where $\scrA^{(\ell)}\equiv \scrA$ and $\scrB^{(\ell)}\equiv \scrB$ for $\ell=0$.

\begin{algorithm}[h]
	\caption{OSAVE: Orthogonal Successive Approximation via Eigenvectors}\label{alg:RiOMDA}
	\begin{small}
		\begin{algorithmic}[1]
			\REQUIRE $\{\Phi_{s,t}\in\bbR^{d_s\times n_t},\,1\le s,\,t\le v\}$,
			$\{\Psi_{s,s}\in\bbR^{d_s\times d_s},\,1\le s\le v\}$,
			integer $1\le k\le\min\{d_1,\ldots,d_{v}\}$;
			
			\ENSURE  $\{P_s\in\bbO^{d_s\times k}\}$, the set of most correlated matrices.
			
			\STATE compute the top eigenvector $[\bq_1^{\T},\bq_2^{\T}\ldots,\bq_{v}^{\T}]^{\T}$ of $\scrA-\lambda\scrB$ by
			Algorithm~\ref{alg:eLOCG}, where $\bq_s\in\bbR^{d_s}$;
			
			\STATE $\bp_s^{(1)}=\bq_s/\|\bq_s\|_2$ for $s=1,2,\ldots,v$;
			
			\FOR{$\ell=1,2\ldots, k-1$}
			
			\STATE compute the top eigenvector $[\bq_1^{\T},\bq_2^{\T}\ldots,\bq_{v}^{\T}]^{\T}$ of
			$\scrA^{(\ell)}-\lambda\scrB^{(\ell)}$ by Algorithm~\ref{alg:eLOCG}, where $\bq_s\in\bbR^{d_s}$;
			
			\STATE $\bp_s^{(\ell+1)}=\bq_s/\|\bq_s\|_2$ for $s=1,2,\ldots,v$;
			\ENDFOR
			\STATE $P_s=[\bp_s^{(1)},\ldots,\bp_s^{(k)}]$ for $s=1,2,\ldots,v$;
			\RETURN $\{P_s\in\bbO^{d_s\times k}\}$.
		\end{algorithmic}
	\end{small}
\end{algorithm}

According to Algorithm~\ref{alg:eLOCG}, the efficiency of Algorithm~\ref{alg:RiOMDA} critically depends on the execution of
matrix-vector products by $\scrA^{(\ell)}$ and $\scrB^{(\ell)}$. Noting that how
$\scrA^{(\ell)}$ and $\scrB^{(\ell)}$ are defined, together with \eqref{eq:Pi-s-PhiPsi-ell-1} and \eqref{eq:Pi-s-PhiPsi-ell-2}, we find that
$$
\scrA^{(\ell)}=\Pi^{(\ell)}\scrA\Pi^{(\ell)},\,\,
\scrB^{(\ell)}=\Pi^{(\ell)}\scrB\Pi^{(\ell)},
$$
where $\Pi^{(\ell)}=\diag(\Pi_1^{\ell},\ldots,\Pi_s^{\ell})$. Thus $\by:=\scrX^{(\ell)}\bx$ where $\scrX$ is either $\scrA$ or $\scrB$ can be done
in three steps:
\begin{subequations}\label{eq:Xx}
	\begin{align}
	\bx&\leftarrow\Pi^{(\ell)}\bx, \label{eq:Xx-1} \\
	\by&\leftarrow\scrX\bx, \label{eq:Xx-2}\\
	\by&\leftarrow\Pi^{(\ell)}\by. \label{eq:Xx-3}
	\end{align}
\end{subequations}
The operations in \eqref{eq:Xx-1} and \eqref{eq:Xx-3} are the same one, and should be implemented as follows. In the case of
\eqref{eq:Xx-1}, write $\bx=[\bx_1^{\T},\ldots,\bx_v^{\T}]^{\T}$ where $\bx_s\in\bbR^{n_s}$ and do
$$
\bx_s\leftarrow\bx_s-P_s^{(\ell)}\big(\big[P_s^{(\ell)}\big]^{\T}\bx_s\big)\,\,\forall s,
$$
where the bracket must be respected for maximum computational efficiency. The operation in \eqref{eq:Xx-2} can be broken
into many mini-ones $\Phi_{s,t}\bx_t$, $\Psi_{s,s}\bx_s$ for all $s,\,t$ whose calculations depend on the structures in
$\Phi_{s,t}$ and $\Psi_{s,s}$ from the underlying task. While it is impossible for us to
offer recommendations on a very general setting, a frequent scenario
where OMvSL is needed has $\Phi_{s,t}$ and $\Psi_{s,s}$  taking the form
\begin{subequations}\label{eq:PhiPsi-form}
	\begin{equation}\label{eq:PhiPsi-form-1}
	\Phi_{s,t}=A_sA_t^{\T}, \quad \Psi_{s,s}=B_sB_s^{\T}
	\end{equation}
	where
	\begin{align}
	A_s&=A_s^{\raw}\left(I_{m_a}-\frac 1{m_a}\bone_{m_a}\bone_{m_a}^{\T}\right)\in\bbR^{d_s\times m_a}, \label{eq:PhiPsi-form-2}\\
	B_s&=B_s^{\raw}\left(I_{m_s}-\frac 1{m_s}\bone_{m_s}\bone_{m_s}^{\T}\right)\in\bbR^{d_s\times m_s}. \label{eq:PhiPsi-form-3}
	\end{align}
\end{subequations}
Here $A_s^{\raw}$ and $B_s^{\raw}$ represent raw input data matrices from an application, which may also be sparse.
In such a  scenario,  $A_s$ and $B_s$ should not be formed explicitly in a large scale application, i.e.,
at least one of $d_s$, $m_a$, and $m_s$ is large, say in the tens of thousands or more, and neither should $\Phi_{s,t}$ and $\Psi_{s,s}$.
As an example, $\by_s:=\Phi_{s,t}\bx_t$ can be executed in the order as follows:
$$
\bz\leftarrow (A_t^{\raw})^{\T}\bx_t,\,
\bz\leftarrow \bz-\frac {\bone_{m_a}^{\T}\bz}{m_a},\,
\by_s\leftarrow A_s^{\raw}\bz.
$$

To get a sense of the computational complexity of OSAVE (Algorithm~\ref{alg:RiOMDA}), in what follows we present a rough estimate, assuming $\Phi_{s,t}$ and $\Psi_{s,s}$ are given and dense. For the $\ell$th loop: lines 3--6 of Algorithm~\ref{alg:RiOMDA}
which calls Algorithm~\ref{alg:eLOCG}, we have, for the leading cost terms for one loop of Algorithm~\ref{alg:eLOCG}
(lines 6--10),
\begin{enumerate}[(a)]
	\item matrix-vector products by $\scrA^{(\ell)}$ and $\scrB^{(\ell)}$: $2n_{\nkry}(d^2+\sum_sd_s^2+8d\ell)$,
	\item orthgonalization in generating $W$: $6dn_{\nkry}$ if by the Lanczos process or $2dn_{\nkry}^2$ if also with full reorthgonalization (recommended),
	\item forming $W^{\T}AW$ and $W^{\T}BW$ (assuming  $AW$ and $BW$ built along the way are reused): $4dn_{\nkry}^2$,
	\item solving $W^{\T}AW-\lambda W^{\T}BW$: $14n_{\nkry}^3$ \cite[p.500]{govl:2013}.
\end{enumerate}
Here $d=\sum_sd_s$ and these estimates work for $\ell=0$, i.e., line 1 of Algorithm~\ref{alg:RiOMDA}, too. For simplicity,
let us assume that on average Algorithm~\ref{alg:eLOCG} takes $m$ iterations to finish, and full reorthgonalization is used
for robustness. Then the overall complexity estimate is
\begin{multline}\label{eq:complx}
m\left\{kn_{\nkry}\left[2d^2+\sum_sd_s^2+6dn_{\nkry}\right]+8n_{\nkry}dk^2\right\} 
\approx 2mkn_{\nkry}d^2,
\end{multline}
where we have dropped the cost in solving $W^{\T}AW-\lambda W^{\T}BW$ due to that $n_{\nkry}$ is usually of $O(1)$,
and we have assumed $k\ll d$ in practice. Further improvement in complexity is possible if  $A_s$ and $B_s$ in
\eqref{eq:PhiPsi-form} are very sparse, and then $d^2$ in \eqref{eq:complx} can be replaced by the total number of nonzero entries
in $A_s$ and $B_s$ for all $s$.

\section{Experiments} \label{sec:experiments}
In this section, we will evaluate the effectiveness of our proposed models instantiated from the unified framework (\ref{op:ouf}) by comparing with existing methods on two learning tasks:  multi-view feature extraction and multi-view multi-label classification.

\begin{landscape}
\begin{table*}[!h]
	\caption{Datasets for feature extraction (followed by classification), where the number of features for each view is shown inside the bracket.}  \label{tab:datasets}
	\vspace{-0.15in}
	\centering
	\begin{scriptsize}
		\begin{tabular}{@{}c|c|c|c|c|c|c|c|c@{}}
			\hline
			Dataset &  samples & class& view 1 &view 2 &view 3 &view 4 &view 5 &view 6 \\
			\hline
			mfeat  & 2000 & 10 &fac (216) & fou (76) & kar (64) & mor (6) & pix (240) & zer (47) \\
			Caltech101-7 & 1474 & 7 & CENTRIST (254) & GIST (512) & LBP (1180) & HOG (1008) & CH (64) & SIFT-SPM (1000)\\
			Caltech101-20 & 2386& 20 & CENTRIST (254) & GIST (512) & LBP (1180) & HOG (1008) & CH (64) & SIFT-SPM (1000)\\
			Scene15 & 4310 & 15 & CENTRIST (254) & GIST (512) & LBP (531)  & HOG (360) & SIFT-SPM (1000) & -\\
			Reuters & 18758 & 6 & English(21531) & France (24892) & German (34251) & Italian (15506) & Spanish (11547) & -\\
			Ads & 3279 & 2 & url+alt+caption (588) & origurl (495) & ancurl (472) & - & - & -\\
			\hline
		\end{tabular}
	\end{scriptsize}
\end{table*}
\end{landscape}

\subsection{Multi-view Feature Extraction} \label{sec:mvfe-exp}

\subsubsection{Datasets}
Six datasets in Table~\ref{tab:datasets} are used to evaluate the performance of the proposed models: OGMA, OMLDA, and OMvMDA in terms of multi-view feature extraction.
We apply various feature descriptors, including CENTRIST \cite{wu2008place}, GIST \cite{oliva2001modeling},
LBP \cite{ojala2002multiresolution}, histogram of oriented gradient (HOG), color histogram (CH),
and SIFT-SPM \cite{lazebnik2006beyond}, to extract features of views for image datasets: Caltech101\footnote{http://www.vision.caltech.edu/Image\_Datasets/Caltech101/}\cite{fei2007learning} and
Scene15\footnote{https://figshare.com/articles/15-Scene\_Image\_Dataset/7007177} \cite{lazebnik2006beyond}.
Note that we drop  CH for Scene15 due to the gray-level images.
Multiple Features (mfeat)\footnote{https://archive.ics.uci.edu/ml/datasets/Multiple+Features}, Internet Advertisements (Ads)\footnote{https://archive.ics.uci.edu/ml/datasets/internet+advertisements}, and Reuters\footnote{https://archive.ics.uci.edu/ml/datasets/Reuters+RCV1+RCV2\\+Multilingual,+Multiview+Text+Categorization+Test+collection} are publicly available from UCI machine learning repository.
The dataset mfeat  contains handwritten numeral
data with six views
including profile correlations (fac), Fourier coefficients of the character shapes (fou), Karhunen-Love coefficients (kar), morphological features (mor),
pixel averages in $2 \times 3$ windows (pix), and Zernike moments (zer).
Ads is used to predict whether or not a given hyperlink (associated with an image) is an advertisement and has three views: features based on the terms in the images URL, caption, and alt text (url+alt+caption), features based on the terms in the URL of the current site (origurl), and  features based on the terms in the anchor URL (ancurl).
Reuters  is a multi-view text categorization dataset  containing feature characteristics of documents originally written in five languages (English, French, German, Italian, and Spanish) and their translations over a common set of six categories (C15, CCAT, E21, ECAT, GCAT, and M11). Only a subset of Reuters, those written in English and their translations in other four languages, is used. As the feature dimension of Reuters is too big to handle by the baseline methods,
a preprocessing step is performed by PCA to keep $500$ features per view.  

\subsubsection{Compared methods}
As shown in Subsection \ref{sec:mvda}, our proposed models, although instantiated from the proposed framework (\ref{op:ouf}), are inspired by some of the existing ones. Hence, the three proposed models have close counterparts via solving  generalized eigenvalue problems. Specifically, the compared methods include
\begin{itemize}
	\item GMA \cite{sharma2012generalized}
	
	\item MLDA and MLDA-m with modifications \cite{sun2015multiview}
	
	\item MvMDA \cite{cao2017generalized}
	
	\item MULDA and MULDA-m with modifications  \cite{sun2015multiview}
	
	\item OGMA: the proposed model instantiated from (\ref{op:ouf}) with
	(\ref{eq:phi_gma})
	
	\item OMLDA: the proposed model instantiated from (\ref{op:ouf}) with 
	(\ref{eq:phi_gma-1}) and (\ref{eq:mlda-osi})
	
	\item OMvMDA: the proposed model instantiated from (\ref{op:ouf}) with 
	(\ref{eq:mvmda}).
	
\end{itemize}
Except for MvMDA and OMvMDA, all  methods share the same trade-off parameter to balance the pairwise correlation and supervised information.  In our experiments, we set $\alpha_{s,t} = \alpha, \forall s \not = t$ so as to reduce the complexity of model selection and tune $\alpha \in \{0.01, 0.1, 1, 10, 100\}$ for proper balance in supervised setting. To prevent the singularity of matrices $\{\Psi_{s,s}\}$, we add a diagonal matrix with a small value, e.g., $10^{-8}$, to
$\Psi_{s,s}\,\forall s$ for all compared methods.

\begin{landscape}
\begin{table*}
	\caption{Means and standard deviations of accuracy  by the 1-nearest neighbor classifier on embeddings  by $9$  methods on $6$ multi-view datasets over 10 random draws from each dataset (10\% training and 90\% testing). } \label{tab:mvfe}
	\centering \vspace{-0.1in}
	\begin{small}
	\begin{tabular}{@{}lcccccc@{}}
		\hline
		method& mfeat & Ads & Scene15 & Caltech101-7 & Caltech101-20 & Reuters\\ \hline
		GMA& 0.9399 $\pm$ 0.0087 & 0.9261 $\pm$ 0.0176 & 0.6166 $\pm$ 0.0120 & 0.9325 $\pm$ 0.0104 & 0.8130 $\pm$ 0.0106 & \textbf{0.8369 $\pm$ 0.0047} \\
		MLDA& 0.9284 $\pm$ 0.0052 & 0.9309 $\pm$ 0.0079 & 0.5468 $\pm$ 0.0137 & 0.9229 $\pm$ 0.0079 & 0.7659 $\pm$ 0.0117 & 0.7911 $\pm$ 0.0050 \\
		MvMDA& 0.9378 $\pm$ 0.0091 & 0.7796 $\pm$ 0.0360 & 0.6088 $\pm$ 0.0146 & 0.9265 $\pm$ 0.0078 & 0.8050 $\pm$ 0.0132 & {0.8367 $\pm$ 0.0041} \\
		MULDA& 0.9523 $\pm$ 0.0046 & 0.9249 $\pm$ 0.0352 & 0.5789 $\pm$ 0.0121 & 0.9265 $\pm$ 0.0083 & 0.8220 $\pm$ 0.0109 & 0.7529 $\pm$ 0.0081 \\
		MLDA-m& 0.9309 $\pm$ 0.0079 & 0.9418 $\pm$ 0.0061 & 0.5699 $\pm$ 0.0120 & 0.8978 $\pm$ 0.0098 & 0.7377 $\pm$ 0.0114 & 0.7939 $\pm$ 0.0051 \\
		MULDA-m& 0.9512 $\pm$ 0.0044 & 0.9282 $\pm$ 0.0362 & 0.5795 $\pm$ 0.0154 & 0.9259 $\pm$ 0.0099 & 0.8217 $\pm$ 0.0058 & 0.7589 $\pm$ 0.0081 \\
		OGMA (proposed)& \textbf{0.9609 $\pm$ 0.0060} & 0.9412 $\pm$ 0.0114 & 0.7359 $\pm$ 0.0156 & \textbf{0.9501 $\pm$ 0.0052} & 0.8600 $\pm$ 0.0103 & 0.8360 $\pm$ 0.0039 \\
		OMLDA (proposed)& 0.9571 $\pm$ 0.0064 & 0.9410 $\pm$ 0.0115 & \textbf{0.7547 $\pm$ 0.0105} & 0.9498 $\pm$ 0.0048 & \textbf{0.8685 $\pm$ 0.0100} & 0.8353 $\pm$ 0.0037 \\
		OMvMDA (proposed)& 0.9599 $\pm$ 0.0063 & \textbf{0.9423 $\pm$ 0.0103} & 0.7198 $\pm$ 0.0191 & 0.9471 $\pm$ 0.0072 & 0.8428 $\pm$ 0.0102 & 0.8347 $\pm$ 0.0037 \\
		\hline
	\end{tabular}
\end{small}
\end{table*}
\end{landscape}

\begin{figure}
	\centering
	\begin{tabular}{@{}c@{}c@{}}
		\includegraphics[width=0.45\textwidth]{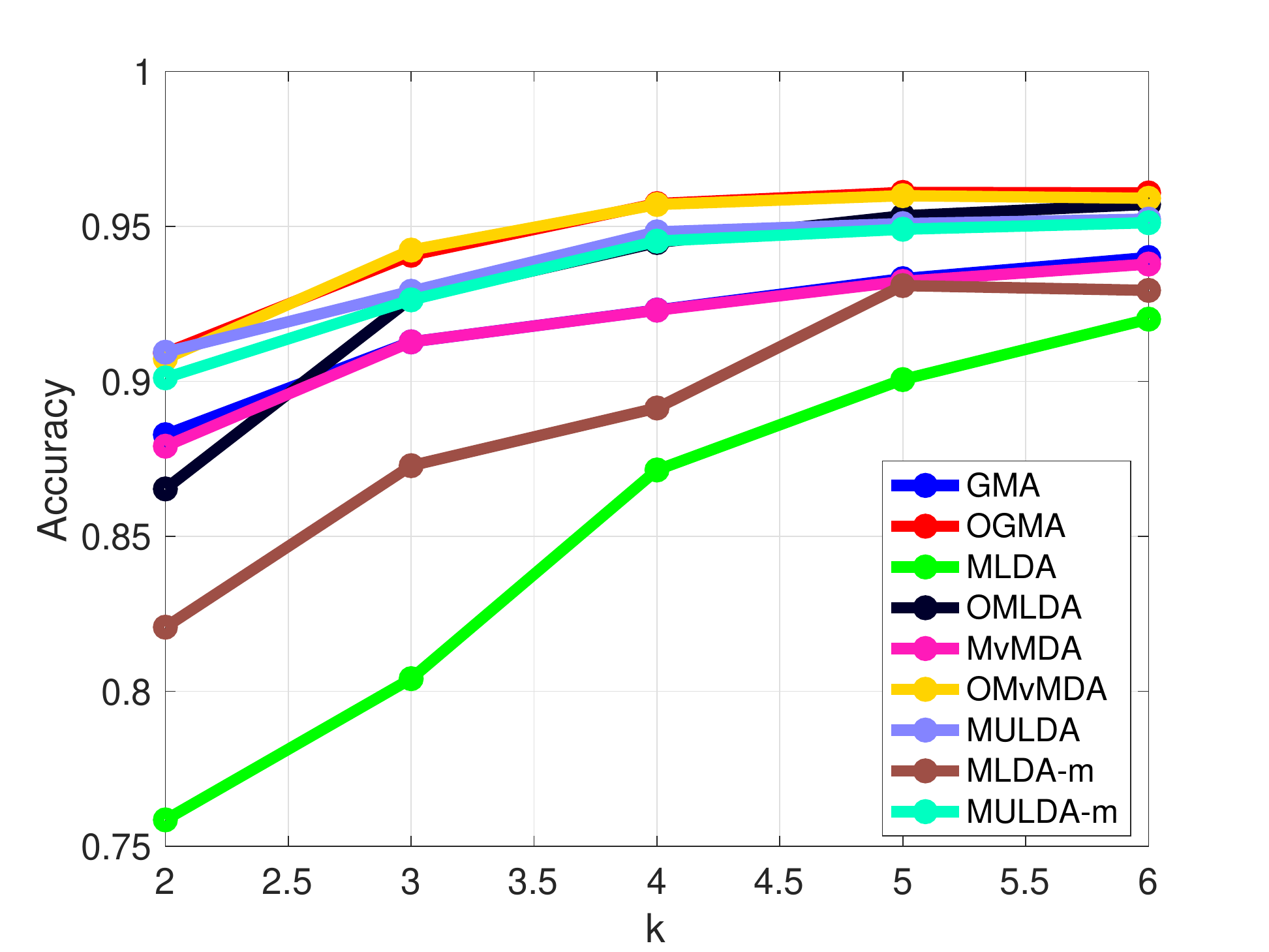} &
		\includegraphics[width=0.45\textwidth]{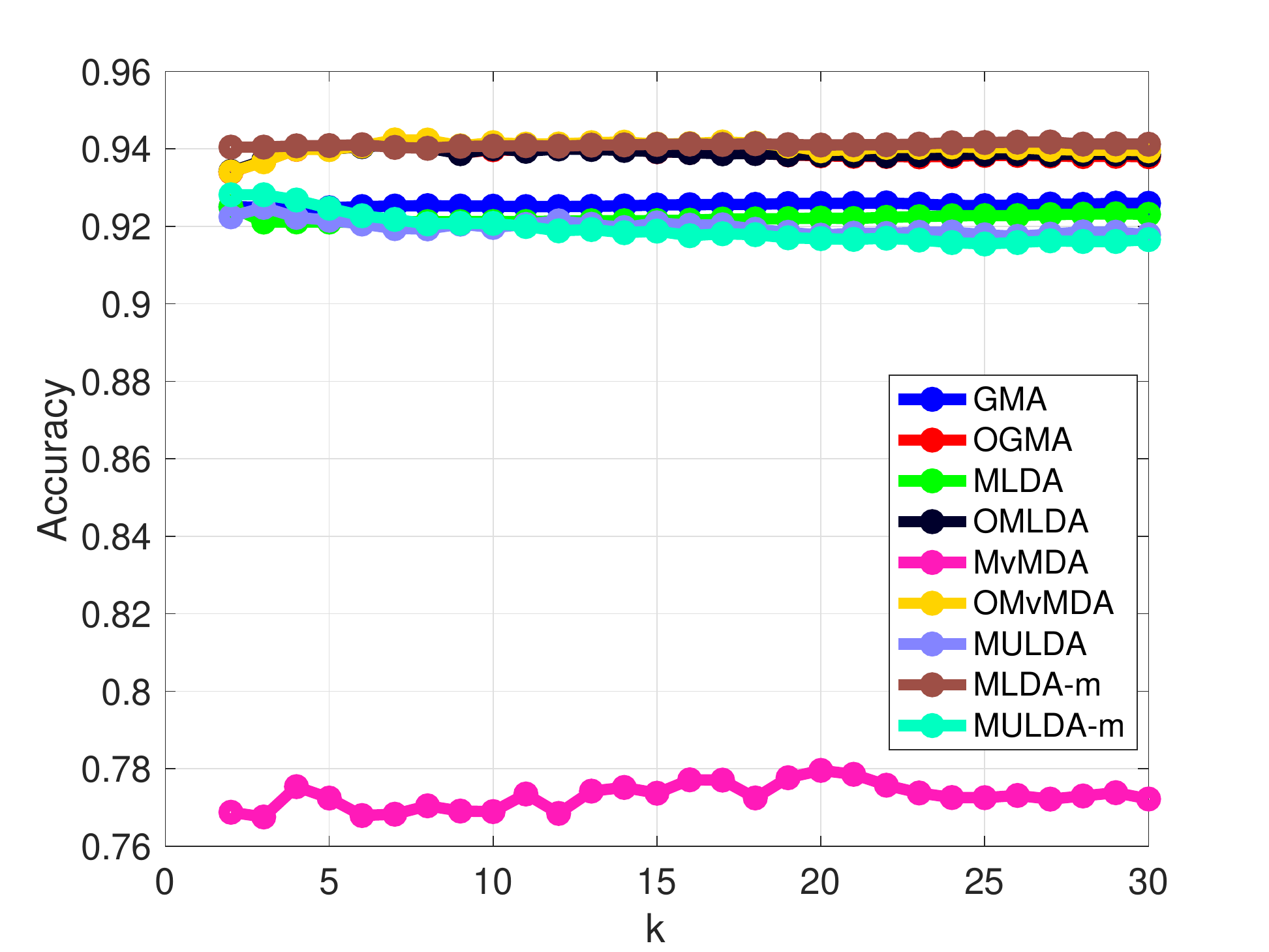} \\
		(a) mfeat & (b) Ads \\
		\includegraphics[width=0.45\textwidth]{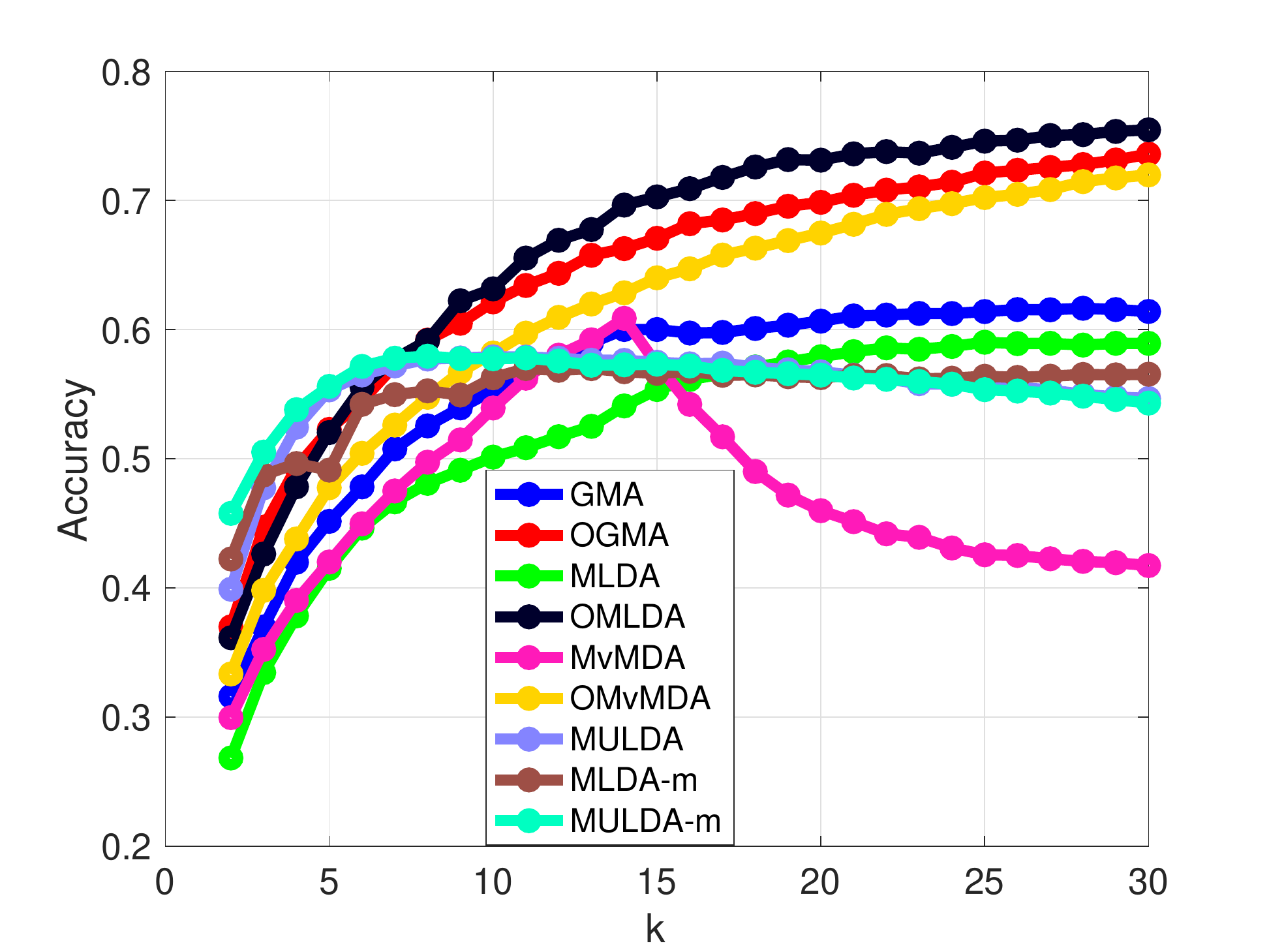} & 
		\includegraphics[width=0.45\textwidth]{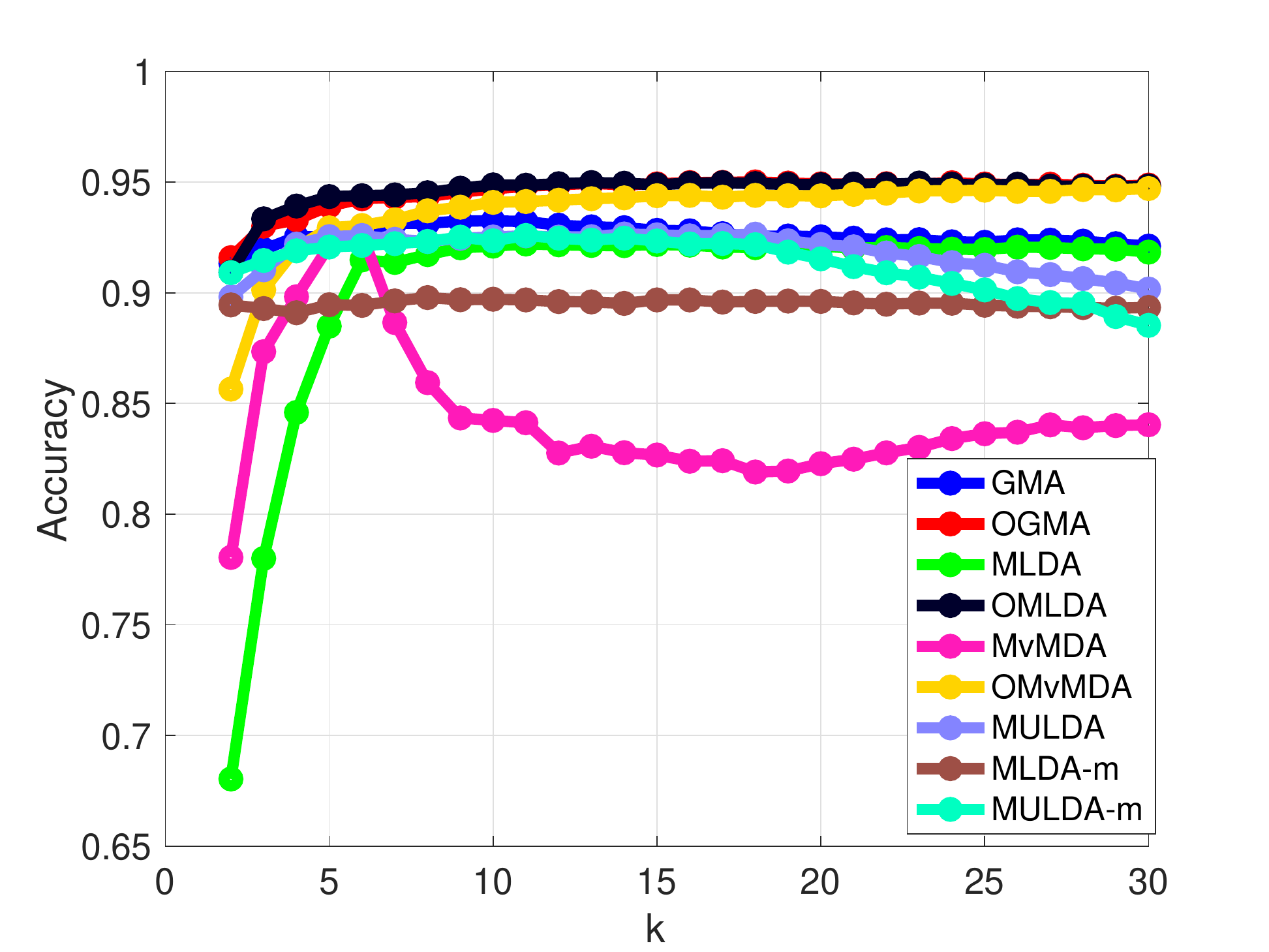} \\
		(c) Scene15 & (d) Caltech101-7\\
		\includegraphics[width=0.45\textwidth]{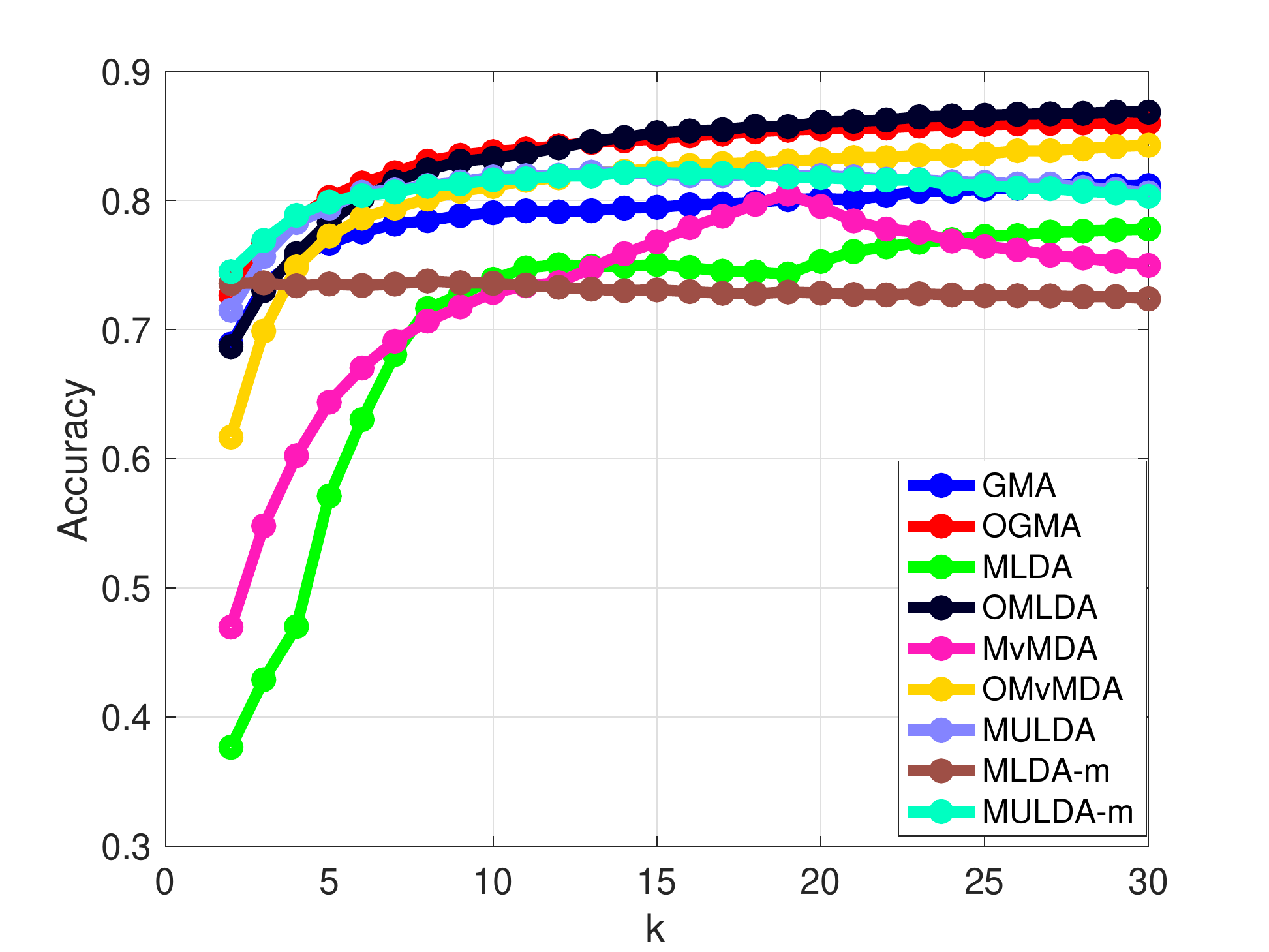} &
		\includegraphics[width=0.45\textwidth]{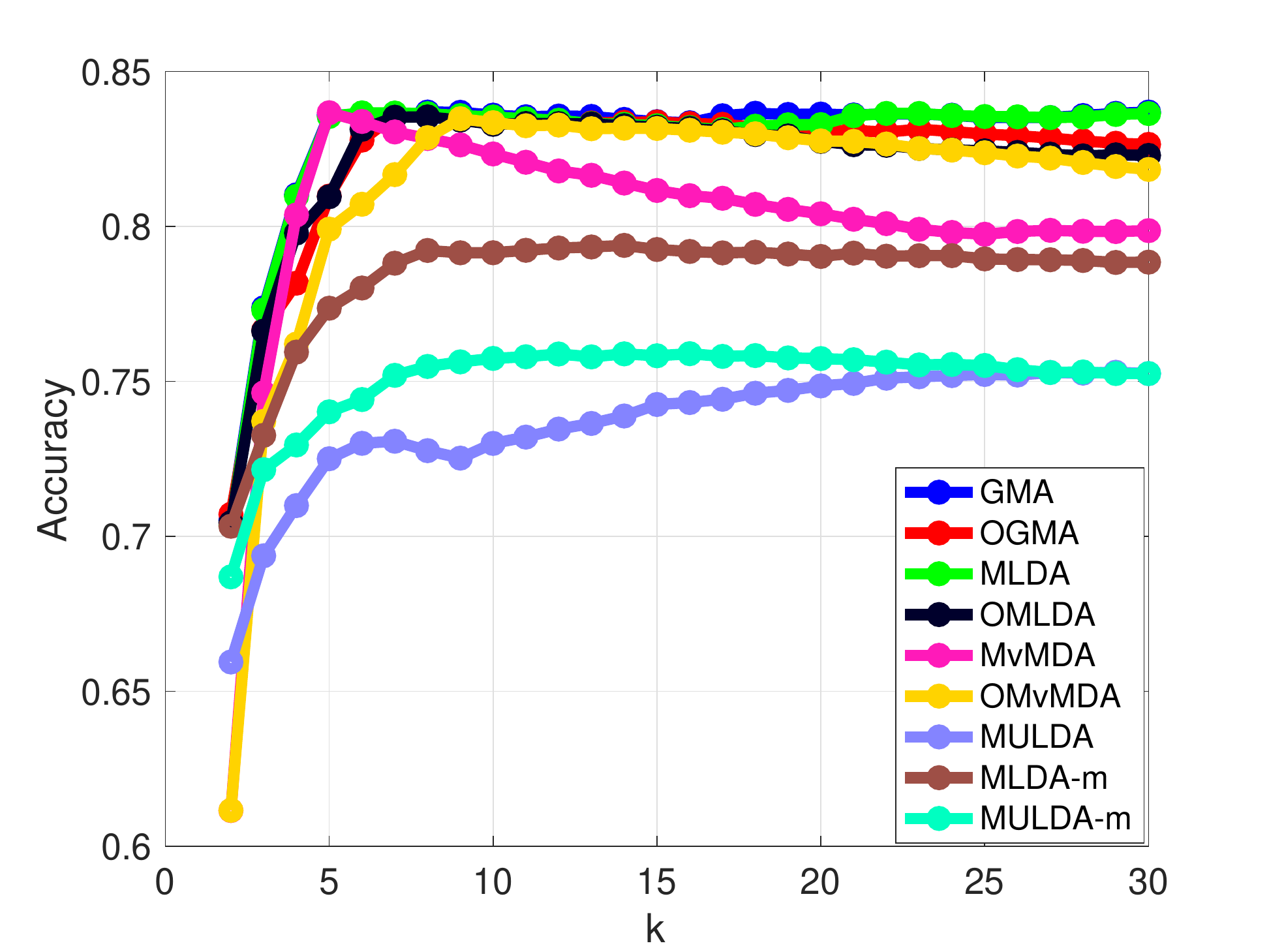}\\
		(e) Caltech101-20 & (f) Reuters\\
	\end{tabular}
	\caption{Classification accuracy of $9$ methods on $6$ datasets over $10$ random splits (10\% training and 90\% testing), as $k$ varies.} \label{fig:mvfs-k}
\end{figure}

\subsubsection{Classification}
To evaluate the learning performance of compared methods, the 1-nearest neighbor classifier as the base classifier is employed. We run each method to learn projection matrices by varying the dimension of the common subspace $k \in [2, 30]$ for all datasets except for  mfeat with $k \in [2, 6]$ due to the smallest view of $6$ features.
We split the data into training and testing with  ratio 10/90.
The learned projection matrices are used to transform both training and testing data into the latent common space, and then classifier is trained and tested in this space. Following  \cite{shen2015orthogonal,Zhang2020Self,wang2020orthogonal}, the serial feature fusion strategy is employed by concatenating projected features from all views.
Classification accuracy is used to measure the learning performance. Experimental results are reported in terms of the average and standard deviation over 10 randomly drawn splits.

Table~\ref{tab:mvfe} shows the best results of $9$ compared methods on $6$ multi-view datasets with $10\%$ training and $90\%$ testing over all tested $k$s and $\alpha$s (the analysis on parameter sensitivity and training sample size will be discussed in subsections \ref{sec:parameer} and \ref{sec:sample-size}, respectively). From Table~\ref{tab:mvfe}, we have the following  observations: (i) our proposed models instantiated from (\ref{op:ouf}) generally outperform their counterparts that solve some relaxed generalized eigenvalue problems. Although GMA produces the best results on Reuters, the differences compared to each of the three proposed methods are all marginal; (ii) three proposed models demonstrate best results on different datasets, while OGMA and OMLDA perform consistently better than OMvMDA on five of six datasets. This empirically shows that the model hypothesis in each model is data-dependent.

\subsubsection{Parameter Sensitivity Analysis} \label{sec:parameer}
The sensitivity analyses on parameters $k$ and $\alpha$ are performed by  varying one of them while recording
the best average accuracy over the other within its testing range.

Figure~\ref{fig:mvfs-k} shows the results of $9$ methods  on six datasets as $k$ varies. Most compared methods demonstrate the increasing trend when $k$ increases. The proposed methods produce consistently better accuracies than others. On Ads,  Caltech101-7 and Reuters, our methods show the saturation on accuracy, while MvMDA shows a significant drop after the certain $k$ on four of six datasets.

\begin{figure}
	\centering
	\begin{tabular}{@{}c@{}c@{}}
		\includegraphics[width=0.45\textwidth]{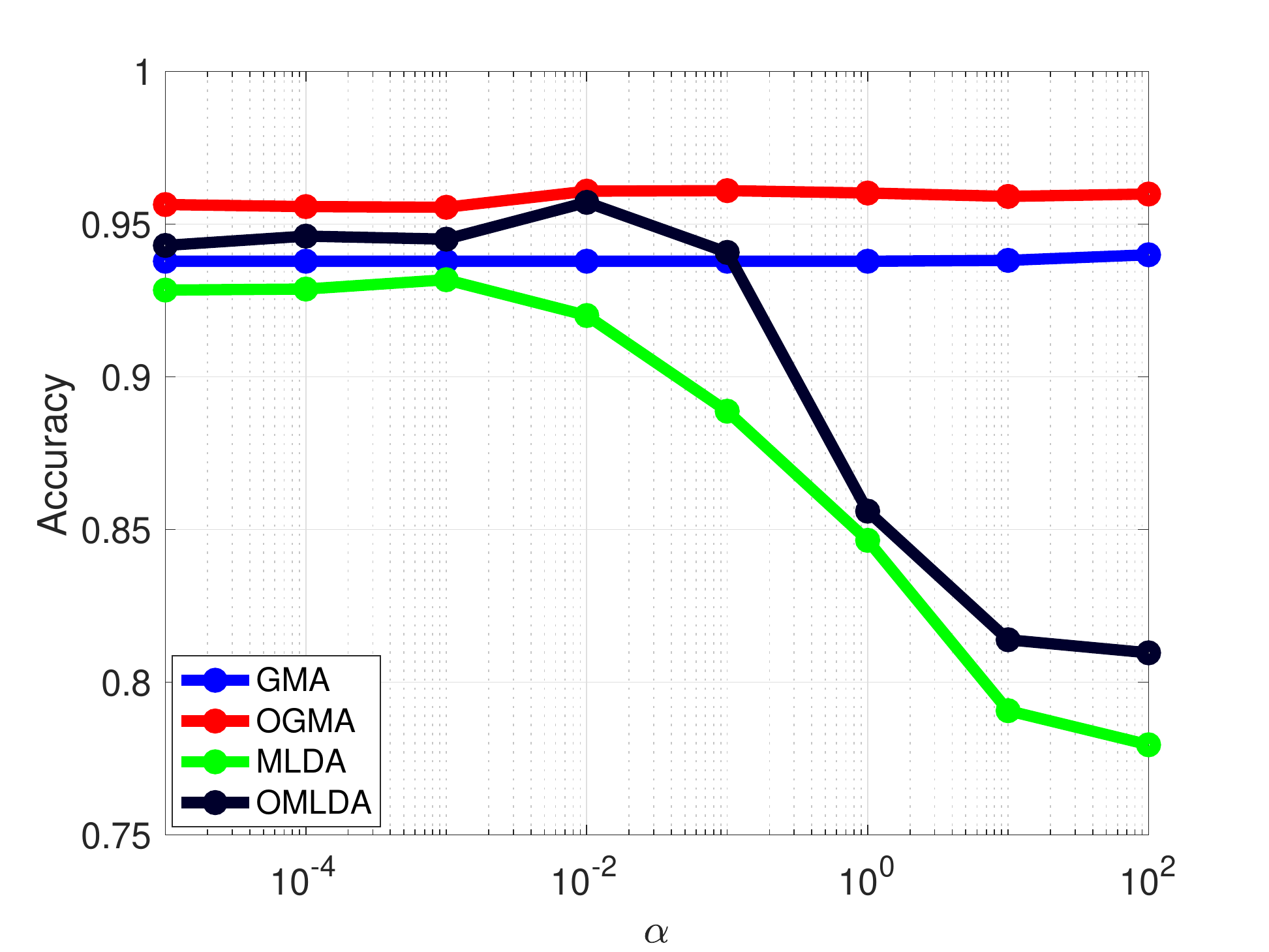} & 
		\includegraphics[width=0.45\textwidth]{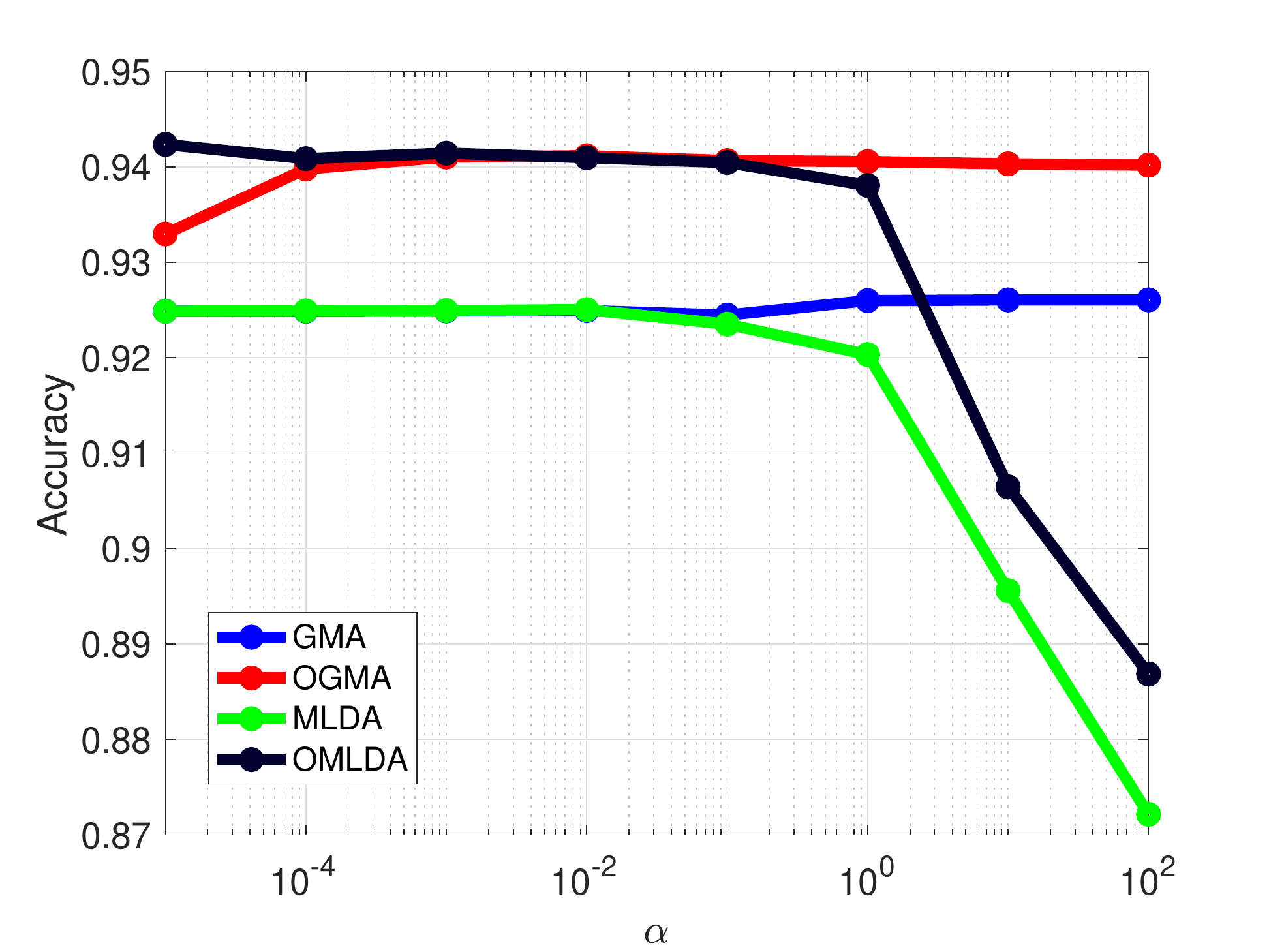} \\
		(a) mfeat & (b) Ads\\
	\end{tabular}
	\caption{Classification accuracy  by $4$ methods on mfeat and Ads  over $10$ random splits (10\% training and 90\% testing), as $\alpha$ varies in $[10^{-5}, 10^2]$.} \label{fig:mvfs-alpha}
\end{figure}

We further investigate the impact of parameter $\alpha$ on GMA, OGMA, MLDA and OMLDA except MvMDA and OMvMDA  since both methods does not contain parameter $\alpha$. In Figure~\ref{fig:mvfs-alpha}, GMA and OGMA demonstrates quite robust to $\alpha$, and the best accuracy can be obtained around $\alpha=10^{-2}$. However, MLDA and OMLDA are quite sensitive to $\alpha$ and the accuracy decreases significantly especially for $\alpha > 0.1$. These observations imply that more contribution from pairwise correlation may hurt MLDA and OMLDA, but no noticeable impact on GMA and OGMA. Over all tested $\alpha$s, our proposed methods outperform their counterparts.

\subsubsection{Impact on Training Sample Size} \label{sec:sample-size}
We further show the impact of training sample size on the compared methods by varying the ratio of training data from $10\%$ to $60\%$. The best average results over $10$ randomly drawn splits are reported. Fig. \ref{fig:mvfs-sample-size} shows the accuracy improves when the training ratio is increasing on Ads and Caltech101-7. It is observed that (i) all methods show better performance when training sample size increases, (ii) our proposed methods show consistently better results than others, and (iii) all methods converge to similar results when training sample size becomes very large except MvMDA.

\begin{figure}
	\begin{tabular}{@{}c@{}c@{}}
		\includegraphics[width=0.45\textwidth]{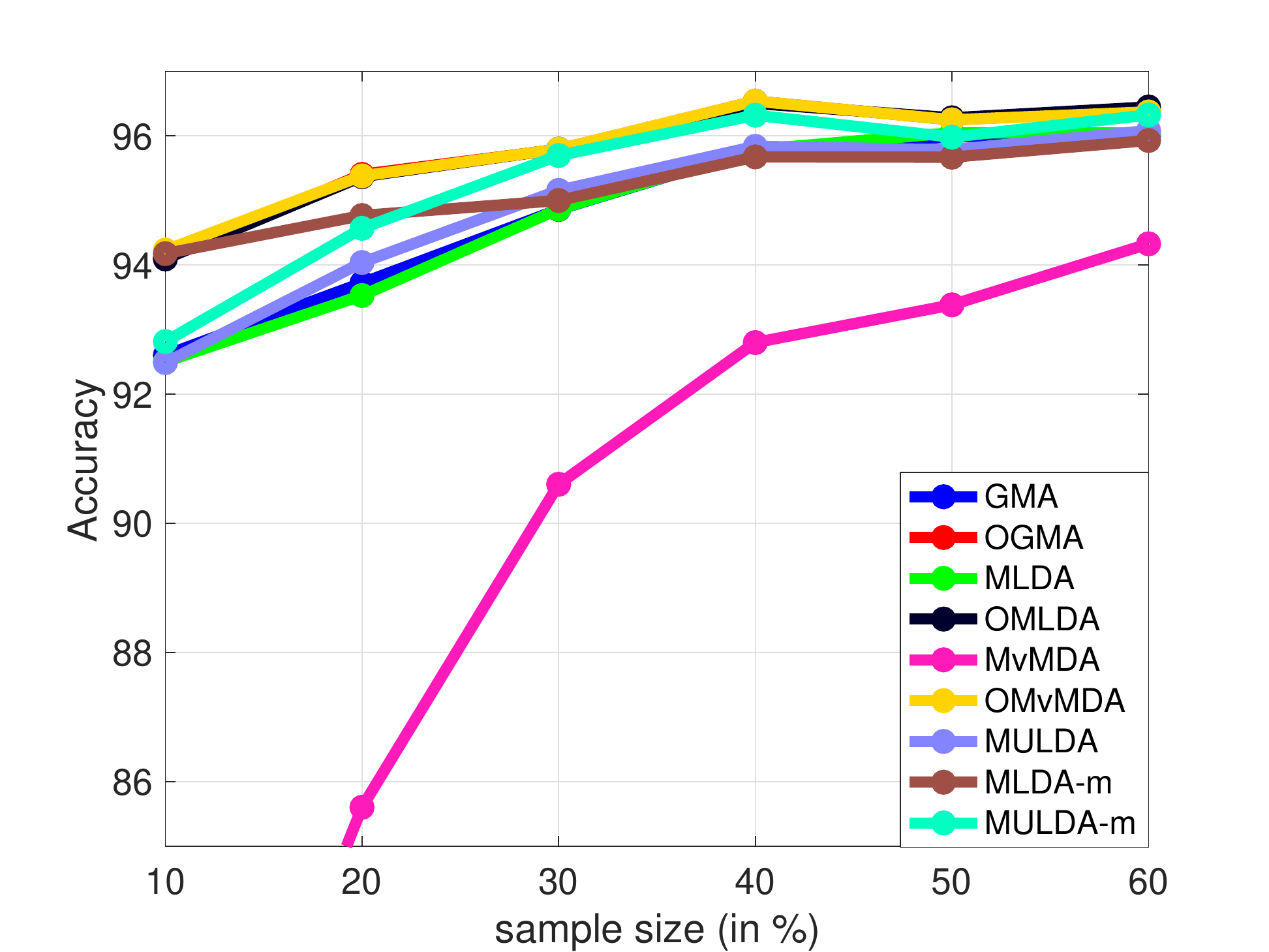} & \includegraphics[width=0.45\textwidth]{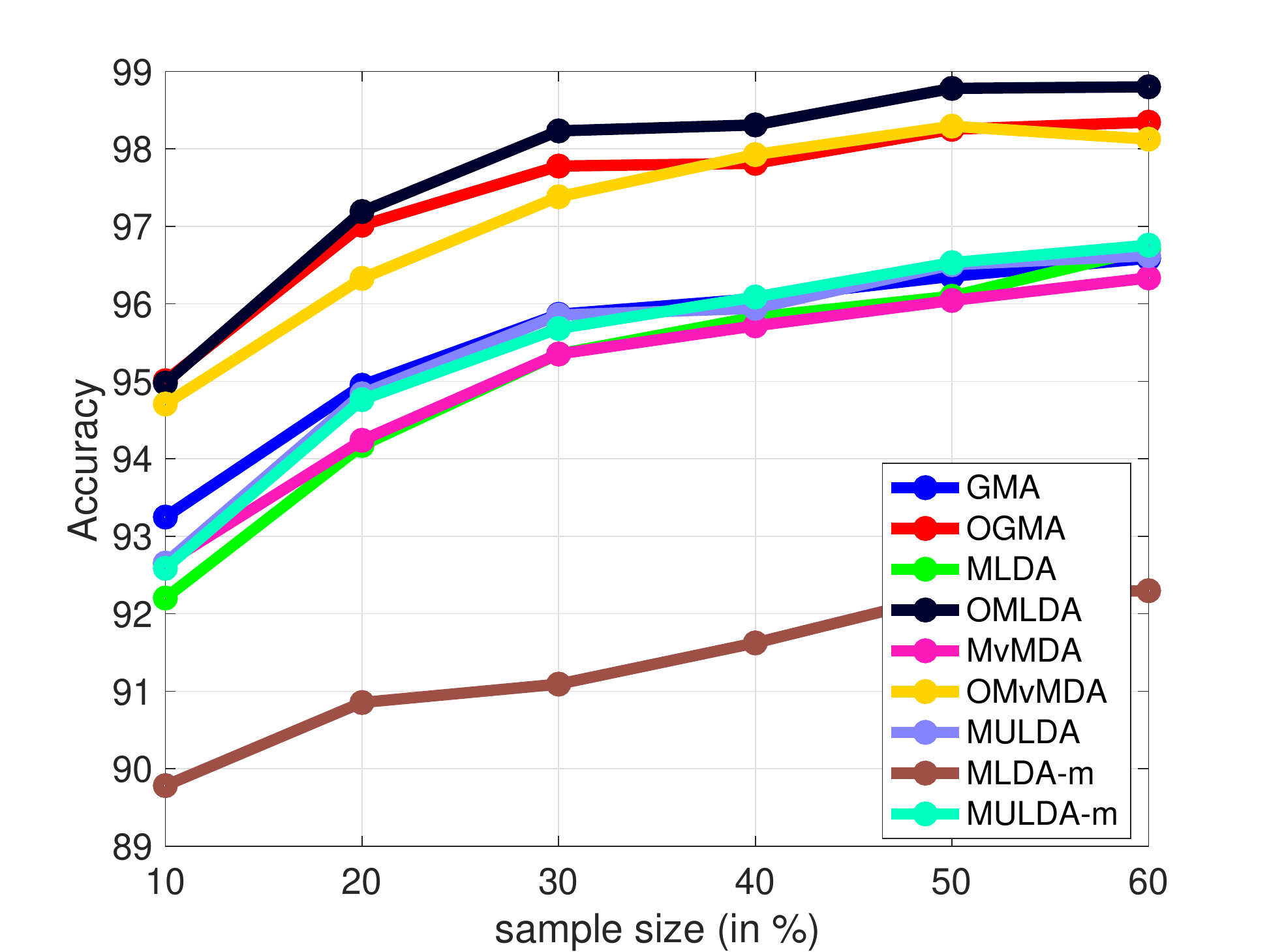}\\
		(a) Ads & (b) Caltech101-7\\
	\end{tabular}
	\caption{Classification accuracy  by all $9$ methods on Ads and Scene15 as the ratio of training data varies from $10\%$ to $60\%$.} \label{fig:mvfs-sample-size}
\end{figure}

\subsection{Multi-view Multi-label Classification}

\begin{table}
	\caption{Multi-view multi-label datasets for classification} \label{tab:mvml}
	\centering \vspace{-0.1in}
	\begin{tabular}{cccc}
		\hline
		& samples & labels & views \\ \hline
		emotions & 593 & 6 & 2 \\
		Corel5k & 4999 & 260 & 7\\
		espgame & 20770 & 268 & 7 \\
		iaprtc12 & 19627 & 291 & 7 \\
		mirflickr & 25000 & 38 & 7 \\
		pascal07 & 9963 & 20 & 7 \\
		\hline
	\end{tabular}
\end{table}

\begin{table*}[!ht]
	\caption{Results in terms of the 5 metrics on the six datasets (10\% for training and 90\% for testing over 10 random splits). Best results are in bold.} \label{tab:mvml-metric}
	\centering \vspace{-0.1in}
	\begin{tiny}
	\begin{tabular}{@{}c|lccccc@{}}
		\hline
		&method& Hamming Loss $\downarrow$ & Ranking Loss $\downarrow$ & One Error $\downarrow$ & Coverage $\downarrow$ &Average Precision $\uparrow$\\ \hline
		
		\multirow{7}{*}{emotions}& view-1& 0.3060 $\pm$ 0.0156& \textbf{0.3038 $\pm$ 0.0195}& 0.4672 $\pm$ 0.0312& 2.4903 $\pm$ 0.1790& 0.6647 $\pm$ 0.1319\\
		& view-2& 0.3403 $\pm$ 0.0247& 0.4392 $\pm$ 0.0173& 0.5949 $\pm$ 0.0422& 3.1069 $\pm$ 0.0625& 0.5678 $\pm$ 0.0625\\
		& concat& 0.3046 $\pm$ 0.0155& 0.3596 $\pm$ 0.0279& 0.4869 $\pm$ 0.0359& 2.8039 $\pm$ 0.1208& 0.6290 $\pm$ 0.1091\\
		& MCCA& 0.3661 $\pm$ 0.0267& 0.4554 $\pm$ 0.0188& 0.6399 $\pm$ 0.0321& 3.1830 $\pm$ 0.1291& 0.5468 $\pm$ 0.1291\\
		& OM$^2$CCA& 0.3006 $\pm$ 0.0124& 0.3249 $\pm$ 0.0346& 0.4948 $\pm$ 0.0488& 2.5740 $\pm$ 0.1779& 0.6492 $\pm$ 0.1777\\
		& HSIC-GEV& 0.3646 $\pm$ 0.0241& 0.4547 $\pm$ 0.0105& 0.6223 $\pm$ 0.0466& 3.0798 $\pm$ 0.1888& 0.5553 $\pm$ 0.1888\\
		&OHSIC& \textbf{0.2953 $\pm$ 0.0110}& 0.3079 $\pm$ 0.0248& \textbf{0.4655 $\pm$ 0.0342}& \textbf{2.4850 $\pm$ 0.1222}& \textbf{0.6662 $\pm$ 0.1222}\\\hline
		
		\multirow{12}{*}{Corel5k}& view-1& 0.0131 $\pm$ 0.0001& 0.1684 $\pm$ 0.0031& 0.7153 $\pm$ 0.0147& 95.3444 $\pm$ 1.4930& 0.2637 $\pm$ 1.5828\\
		& view-2& 0.0131 $\pm$ 0.0001& 0.1672 $\pm$ 0.0034& 0.7031 $\pm$ 0.0110& 94.9287 $\pm$ 1.6843& 0.2689 $\pm$ 1.6548\\
		& view-3& 0.0131 $\pm$ 0.0001& 0.1664 $\pm$ 0.0033& 0.6606 $\pm$ 0.0072& 95.3894 $\pm$ 1.7478& 0.2862 $\pm$ 1.7478\\
		& view-4& 0.0131 $\pm$ 0.0000& 0.1724 $\pm$ 0.0030& 0.7187 $\pm$ 0.0154& 97.7932 $\pm$ 1.6951& 0.2592 $\pm$ 1.6252\\
		& view-5& 0.0131 $\pm$ 0.0000& 0.1709 $\pm$ 0.0027& 0.7366 $\pm$ 0.0107& 96.2485 $\pm$ 1.4062& 0.2502 $\pm$ 1.4062\\
		& view-6& 0.0131 $\pm$ 0.0000& 0.1707 $\pm$ 0.0025& 0.7365 $\pm$ 0.0137& 96.2007 $\pm$ 1.3439& 0.2520 $\pm$ 1.3704\\
		& view-7& 0.0131 $\pm$ 0.0000& 0.1691 $\pm$ 0.0026& 0.6906 $\pm$ 0.0065& 96.3108 $\pm$ 1.3974& 0.2716 $\pm$ 1.5536\\
		& concat& 0.0131 $\pm$ 0.0001& \textbf{0.1597 $\pm$ 0.0040}& 0.6591 $\pm$ 0.0135& 92.5057 $\pm$ 2.1126& 0.2999 $\pm$ 2.1126\\
		& MCCA& 0.0131 $\pm$ 0.0000& 0.2013 $\pm$ 0.0020& 0.7799 $\pm$ 0.0115& 104.9648 $\pm$ 1.4837& 0.2121 $\pm$ 1.4837\\
		& OM$^2$CCA&\textbf{0.0130 $\pm$ 0.0000}& 0.1668 $\pm$ 0.0028& 0.6982 $\pm$ 0.0106& 94.7535 $\pm$ 1.4380& 0.2729 $\pm$ 1.4651\\
		& HSIC-GEV& 0.0131 $\pm$ 0.0000& 0.1933 $\pm$ 0.0031& 0.7885 $\pm$ 0.0161& 104.6444 $\pm$ 1.5763& 0.2011 $\pm$ 1.6329\\
		& OHSIC& \textbf{0.0130 $\pm$ 0.0001}& 0.1601 $\pm$ 0.0026& \textbf{0.6374 $\pm$ 0.0126}& \textbf{91.8414 $\pm$ 1.5051}& \textbf{0.3022 $\pm$ 1.3774}\\
		\hline
		\multirow{12}{*}{iaprtc12}
		& view-1& 0.0196 $\pm$ 0.0000& 0.1871 $\pm$ 0.0012& 0.6746 $\pm$ 0.0047& 142.9013 $\pm$ 0.9209& 0.2216 $\pm$ 0.8681\\
		& view-2& 0.0196 $\pm$ 0.0000& 0.1850 $\pm$ 0.0014& 0.6611 $\pm$ 0.0041& 141.9732 $\pm$ 1.0396& 0.2272 $\pm$ 1.1048\\
		& view-3& \textbf{0.0195 $\pm$ 0.0000}& 0.1738 $\pm$ 0.0012& 0.6262 $\pm$ 0.0066& 137.2026 $\pm$ 0.6685& 0.2535 $\pm$ 0.6685\\
		& view-4& 0.0196 $\pm$ 0.0000& 0.1768 $\pm$ 0.0009& 0.6375 $\pm$ 0.0024& 138.5784 $\pm$ 0.6618& 0.2508 $\pm$ 0.6618\\
		& view-5& 0.0197 $\pm$ 0.0000& 0.1879 $\pm$ 0.0010& 0.6902 $\pm$ 0.0039& 143.2122 $\pm$ 0.7937& 0.2179 $\pm$ 0.7937\\
		& view-6& 0.0197 $\pm$ 0.0000& 0.1862 $\pm$ 0.0013& 0.6802 $\pm$ 0.0033& 142.4131 $\pm$ 0.9414& 0.2233 $\pm$ 0.9414\\
		& view-7& 0.0196 $\pm$ 0.0000& 0.1720 $\pm$ 0.0013& 0.6341 $\pm$ 0.0049& 136.0222 $\pm$ 0.7975& 0.2587 $\pm$ 0.7975\\
		& concat& \textbf{0.0195 $\pm$ 0.0000}& 0.1696 $\pm$ 0.0010& 0.6218 $\pm$ 0.0032& 134.6660 $\pm$ 0.7990& 0.2649 $\pm$ 0.8544\\
		& MCCA& 0.0196 $\pm$ 0.0000& 0.1804 $\pm$ 0.0109& 0.6447 $\pm$ 0.0252& 140.0515 $\pm$ 5.4482& 0.2400 $\pm$ 5.0702\\
		& OM$^2$CCA& 0.0196 $\pm$ 0.0000& 0.1709 $\pm$ 0.0011& 0.6220 $\pm$ 0.0030& 135.2006 $\pm$ 0.8948& 0.2559 $\pm$ 0.9850\\
		& HSIC-GEV& \textbf{0.0195 $\pm$ 0.0001}& \textbf{0.1648 $\pm$ 0.0022}& \textbf{0.5893 $\pm$ 0.0035}& \textbf{132.1792 $\pm$ 1.2844}& \textbf{0.2776 $\pm$ 1.2844}\\
		& OHSIC& \textbf{0.0195 $\pm$ 0.0000}& 0.1673 $\pm$ 0.0009& 0.6078 $\pm$ 0.0025& 133.6776 $\pm$ 0.7331& 0.2661 $\pm$ 0.7209\\
		
		\hline
		
		\multirow{12}{*}{espgame} & view-1& \textbf{0.0174 $\pm$ 0.0000}& 0.2150 $\pm$ 0.0011& 0.6762 $\pm$ 0.0052& 134.8974 $\pm$ 0.5372& 0.2235 $\pm$ 0.5178\\
		& view-2& \textbf{0.0174 $\pm$ 0.0000}& 0.2144 $\pm$ 0.0013& 0.6766 $\pm$ 0.0058& 134.6899 $\pm$ 0.6207& 0.2238 $\pm$ 0.6207\\
		& view-3& 0.0175 $\pm$ 0.0000& 0.2035 $\pm$ 0.0009& 0.7213 $\pm$ 0.0049& 129.8373 $\pm$ 0.4775& 0.2185 $\pm$ 0.6298\\
		& view-4& 0.0175 $\pm$ 0.0000& 0.2030 $\pm$ 0.0012& 0.7169 $\pm$ 0.0032& 129.1738 $\pm$ 0.7794& 0.2201 $\pm$ 0.7794\\
		& view-5& \textbf{0.0174 $\pm$ 0.0000}& 0.2157 $\pm$ 0.0009& 0.6668 $\pm$ 0.0051& 135.3101 $\pm$ 0.4779& 0.2262 $\pm$ 0.4854\\
		& view-6& \textbf{0.0174 $\pm$ 0.0000}& 0.2159 $\pm$ 0.0009& 0.6687 $\pm$ 0.0033& 135.4435 $\pm$ 0.4592& 0.2252 $\pm$ 0.5143\\
		& view-7& 0.0175 $\pm$ 0.0000& 0.2054 $\pm$ 0.0006& 0.7279 $\pm$ 0.0049& 130.7208 $\pm$ 0.5104& 0.2160 $\pm$ 0.5104\\
		& concat& 0.0175 $\pm$ 0.0000& \textbf{0.2015 $\pm$ 0.0010}& 0.6989 $\pm$ 0.0063& \textbf{128.8904 $\pm$ 0.6606}& 0.2283 $\pm$ 0.6859\\
		& MCCA& \textbf{0.0174 $\pm$ 0.0001}& 0.2136 $\pm$ 0.0061& 0.6784 $\pm$ 0.0518& 134.1460 $\pm$ 1.9614& 0.2249 $\pm$ 1.9531\\
		& OM$^2$CCA& \textbf{0.0174 $\pm$ 0.0000}& 0.2076 $\pm$ 0.0008& 0.6283 $\pm$ 0.0040& 132.0874 $\pm$ 0.5329& 0.2454 $\pm$ 0.5074\\
		& HSIC-GEV& \textbf{0.0174 $\pm$ 0.0000}& 0.2068 $\pm$ 0.0010& 0.6236 $\pm$ 0.0053& 131.9247 $\pm$ 0.6241& 0.2481 $\pm$ 0.6241\\
		& OHSIC& \textbf{0.0174 $\pm$ 0.0000}& 0.2061 $\pm$ 0.0010& \textbf{0.6207 $\pm$ 0.0053}& 131.5208 $\pm$ 0.5965& \textbf{0.2495 $\pm$ 0.5965}\\
		\hline
		\multirow{12}{*}{mirflickr} & view-1& 0.1224 $\pm$ 0.0004& 0.1798 $\pm$ 0.0014& 0.4960 $\pm$ 0.0041& 15.3929 $\pm$ 0.0753& 0.4993 $\pm$ 0.0734\\
		& view-2& 0.1220 $\pm$ 0.0002& 0.1780 $\pm$ 0.0014& 0.4872 $\pm$ 0.0033& 15.3257 $\pm$ 0.0511& 0.5054 $\pm$ 0.0613\\
		& view-3& 0.1177 $\pm$ 0.0004& 0.1614 $\pm$ 0.0012& 0.4259 $\pm$ 0.0031& 14.5379 $\pm$ 0.0488& 0.5451 $\pm$ 0.0488\\
		& view-4& 0.1186 $\pm$ 0.0003& 0.1672 $\pm$ 0.0013& 0.4414 $\pm$ 0.0046& 14.8891 $\pm$ 0.0513& 0.5329 $\pm$ 0.0513\\
		& view-5& 0.1227 $\pm$ 0.0004& 0.1815 $\pm$ 0.0013& 0.5053 $\pm$ 0.0040& 15.4924 $\pm$ 0.0535& 0.4933 $\pm$ 0.0641\\
		& view-6& 0.1226 $\pm$ 0.0003& 0.1805 $\pm$ 0.0011& 0.5027 $\pm$ 0.0045& 15.4430 $\pm$ 0.0631& 0.4963 $\pm$ 0.0631\\
		& view-7& 0.1171 $\pm$ 0.0006& 0.1592 $\pm$ 0.0016& 0.4263 $\pm$ 0.0036& 14.3775 $\pm$ 0.0864& 0.5471 $\pm$ 0.0864\\
		& concat& 0.1170 $\pm$ 0.0003& 0.1617 $\pm$ 0.0015& 0.4192 $\pm$ 0.0035& 14.5937 $\pm$ 0.0706& 0.5470 $\pm$ 0.0667\\
		& MCCA& 0.1176 $\pm$ 0.0006& 0.1612 $\pm$ 0.0017& 0.4178 $\pm$ 0.0061& 14.5641 $\pm$ 0.0847& 0.5483 $\pm$ 0.0847\\
		& OM$^2$CCA& 0.1181 $\pm$ 0.0005& 0.1626 $\pm$ 0.0016& 0.4202 $\pm$ 0.0052& 14.6236 $\pm$ 0.0982& 0.5482 $\pm$ 0.1014\\
		& HSIC-GEV& \textbf{0.1131 $\pm$ 0.0008}& \textbf{0.1507 $\pm$ 0.0011}& \textbf{0.3460 $\pm$ 0.0024}& \textbf{13.9975 $\pm$ 0.0970}& \textbf{0.5868 $\pm$ 0.128}0\\
		& OHSIC& 0.1169 $\pm$ 0.0003& 0.1586 $\pm$ 0.0014& 0.4127 $\pm$ 0.0058& 14.4238 $\pm$ 0.0572& 0.5530 $\pm$ 0.0572\\
		\hline
		\multirow{12}{*}{pascal07} & view-1& 0.0730 $\pm$ 0.0005& 0.2786 $\pm$ 0.0049& 0.5946 $\pm$ 0.0029& 6.9247 $\pm$ 0.1447& 0.4425 $\pm$ 0.1447\\
		& view-2& 0.0729 $\pm$ 0.0002& 0.2708 $\pm$ 0.0046& 0.5950 $\pm$ 0.0031& 6.7332 $\pm$ 0.1332& 0.4466 $\pm$ 0.1332\\
		& view-3& 0.0715 $\pm$ 0.0005& 0.2373 $\pm$ 0.0033& 0.5819 $\pm$ 0.0044& 5.9969 $\pm$ 0.1021& 0.4800 $\pm$ 0.0754\\
		& view-4& 0.0702 $\pm$ 0.0003& 0.2328 $\pm$ 0.0041& 0.5656 $\pm$ 0.0042& 5.8909 $\pm$ 0.0956& 0.4928 $\pm$ 0.0956\\
		& view-5& 0.0716 $\pm$ 0.0004& 0.2714 $\pm$ 0.0031& 0.5941 $\pm$ 0.0026& 6.7623 $\pm$ 0.0936& 0.4482 $\pm$ 0.0936\\
		& view-6& 0.0719 $\pm$ 0.0006& 0.2692 $\pm$ 0.0042& 0.5945 $\pm$ 0.0022& 6.7054 $\pm$ 0.0993& 0.4498 $\pm$ 0.0993\\
		& view-7& 0.0699 $\pm$ 0.0005& 0.2219 $\pm$ 0.0032& 0.5617 $\pm$ 0.0049& 5.6492 $\pm$ 0.0718& 0.5006 $\pm$ 0.0718\\
		& concat& 0.0700 $\pm$ 0.0003& 0.2268 $\pm$ 0.0045& 0.5634 $\pm$ 0.0061& 5.7465 $\pm$ 0.1130& 0.4996 $\pm$ 0.1130\\
		& MCCA& 0.0691 $\pm$ 0.0002& 0.2183 $\pm$ 0.0029& 0.5700 $\pm$ 0.0054& 5.5241 $\pm$ 0.0599& 0.4991 $\pm$ 0.0599\\
		& OM$^2$CCA& 0.0694 $\pm$ 0.0003& 0.2179 $\pm$ 0.0037& 0.5723 $\pm$ 0.0060& 5.5003 $\pm$ 0.0788& 0.4960 $\pm$ 0.0714\\
		& HSIC-GEV& \textbf{0.0678 $\pm$ 0.0004}& 0.2185 $\pm$ 0.0052& \textbf{0.5569 $\pm$ 0.004}6& 5.4652 $\pm$ 0.1018& \textbf{0.5088 $\pm$ 0.1018}\\
		& OHSIC& \textbf{0.0678 $\pm$ 0.0004}& \textbf{0.2122 $\pm$ 0.0034}& 0.5604 $\pm$ 0.0046& \textbf{5.3753 $\pm$ 0.0679}& 0.5073 $\pm$ 0.0525\\
		\hline
	\end{tabular}
	\end{tiny}
\end{table*}

\begin{figure*}
	\centering
	\begin{tabular}{@{}c@{}c@{}}
		\includegraphics[width=0.4\textwidth]{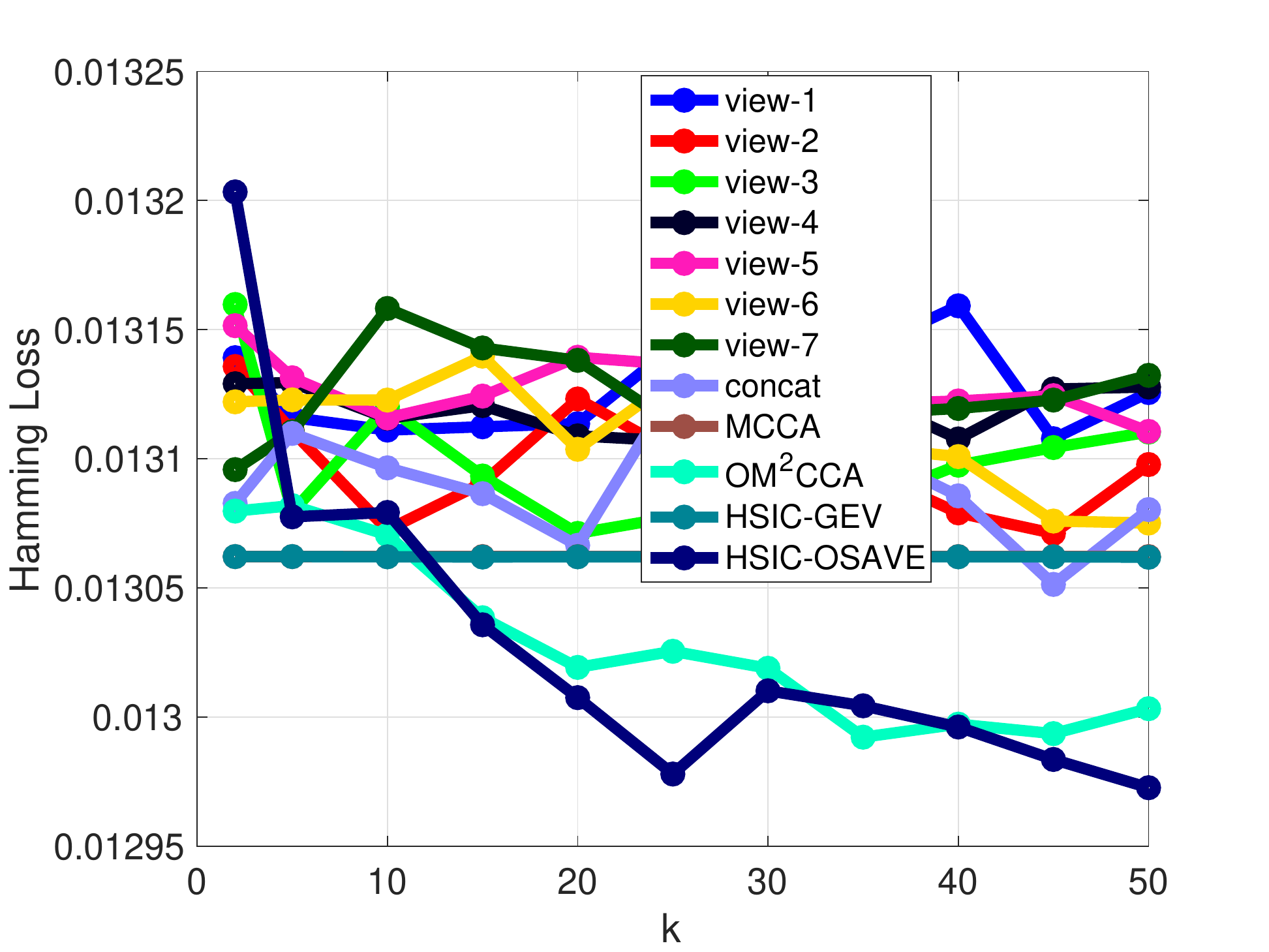} & 		\includegraphics[width=0.4\textwidth]{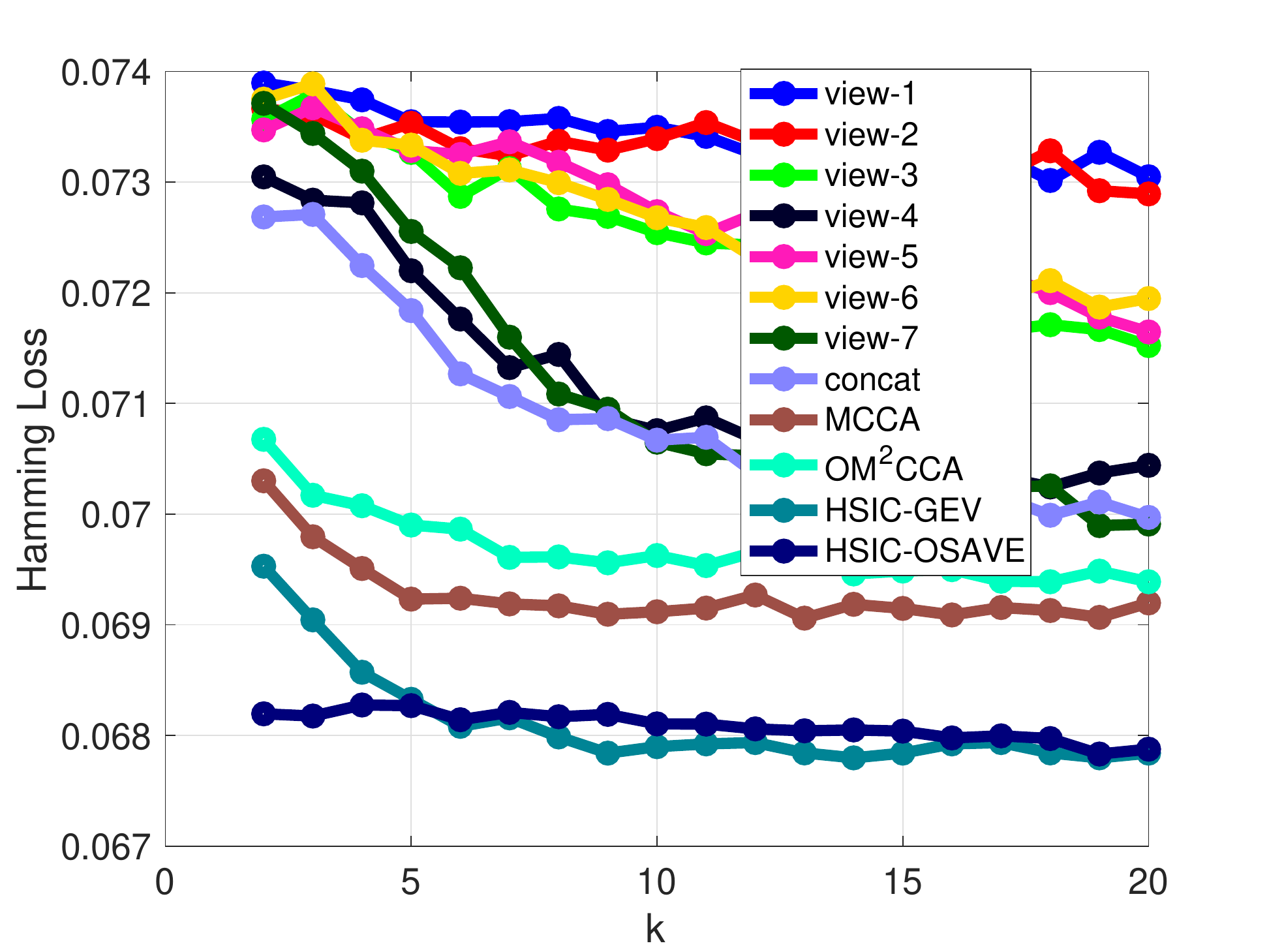}\\
		\includegraphics[width=0.4\textwidth]{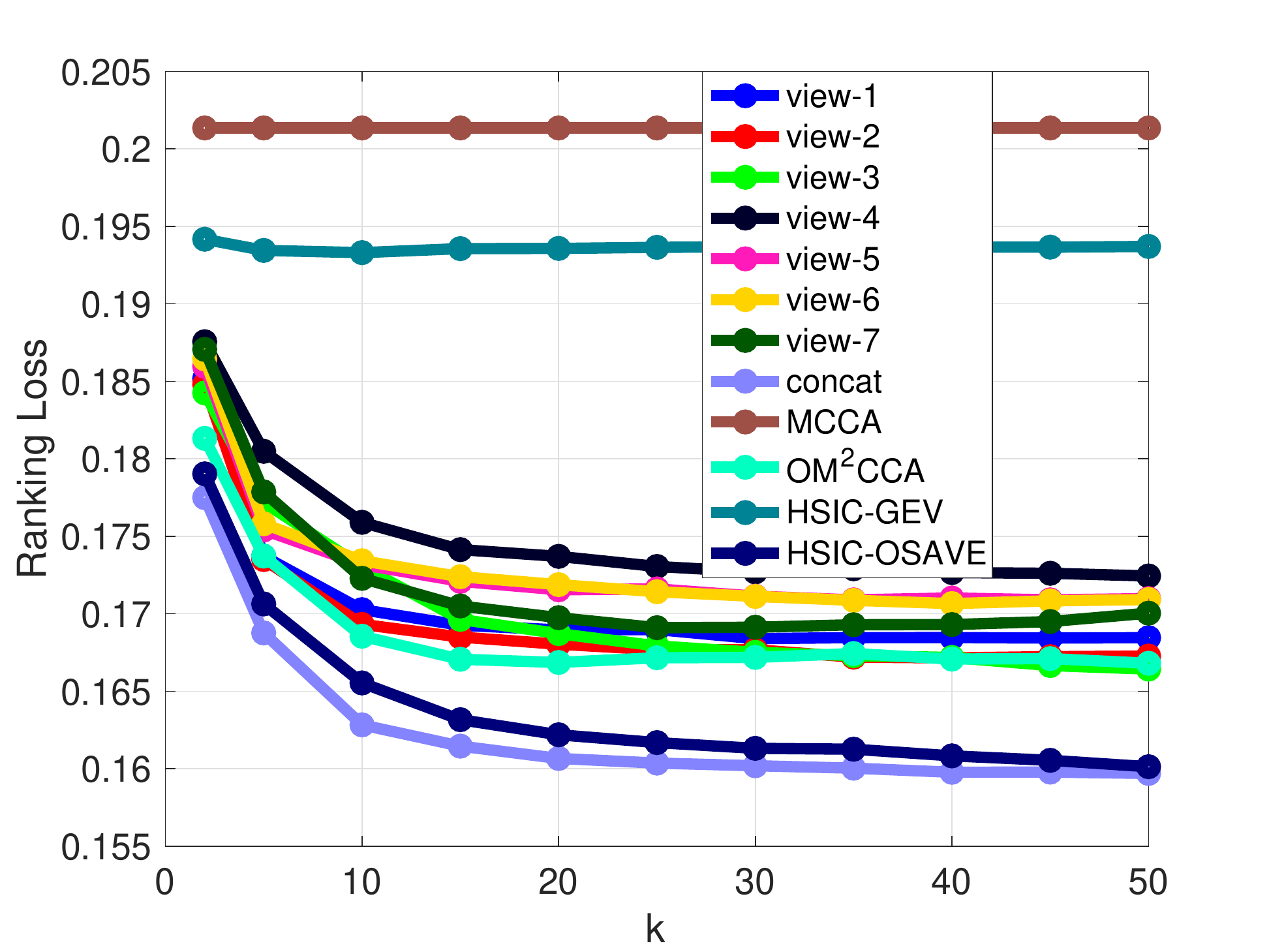} & 		\includegraphics[width=0.4\textwidth]{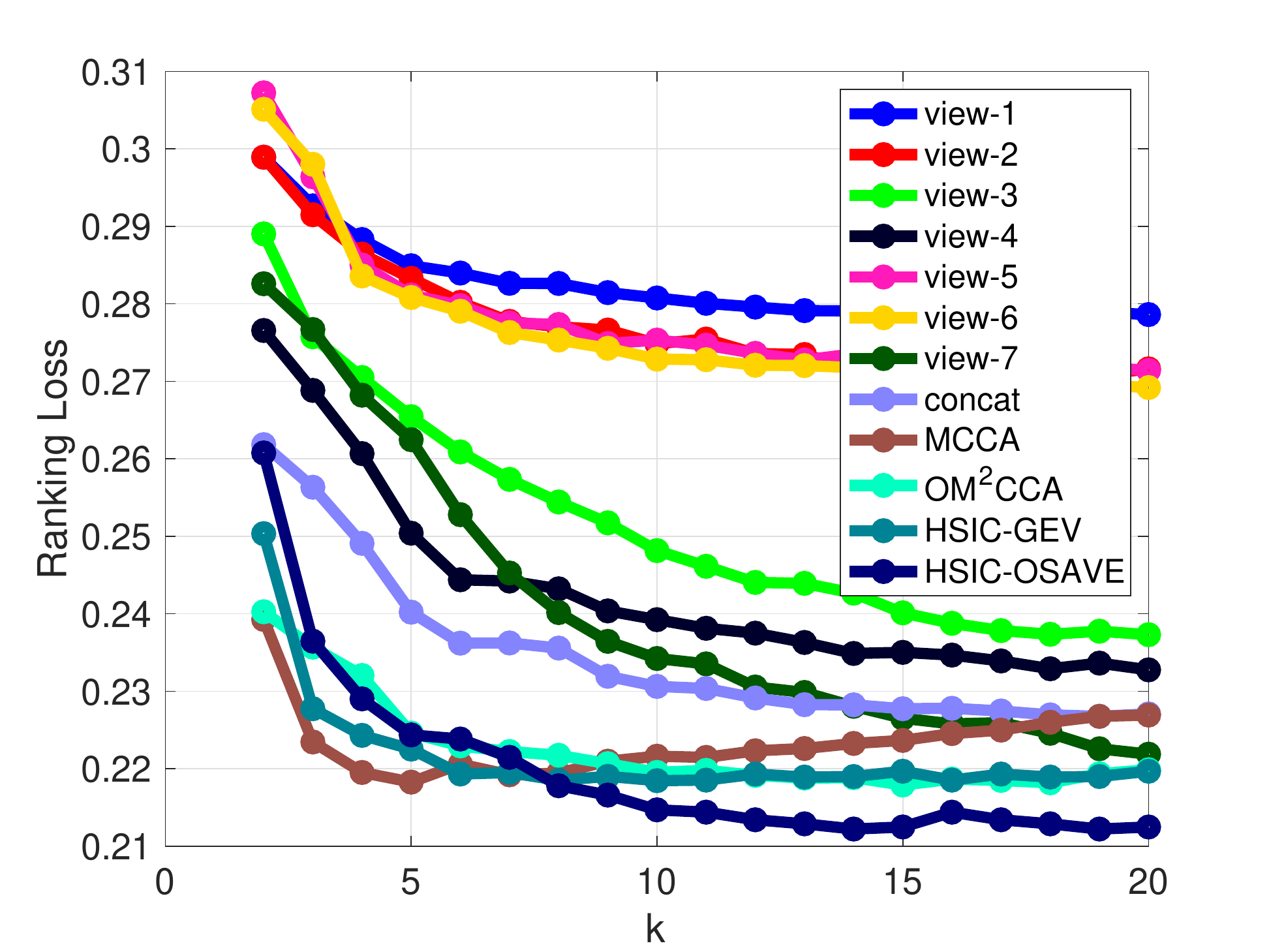}\\
		\includegraphics[width=0.4\textwidth]{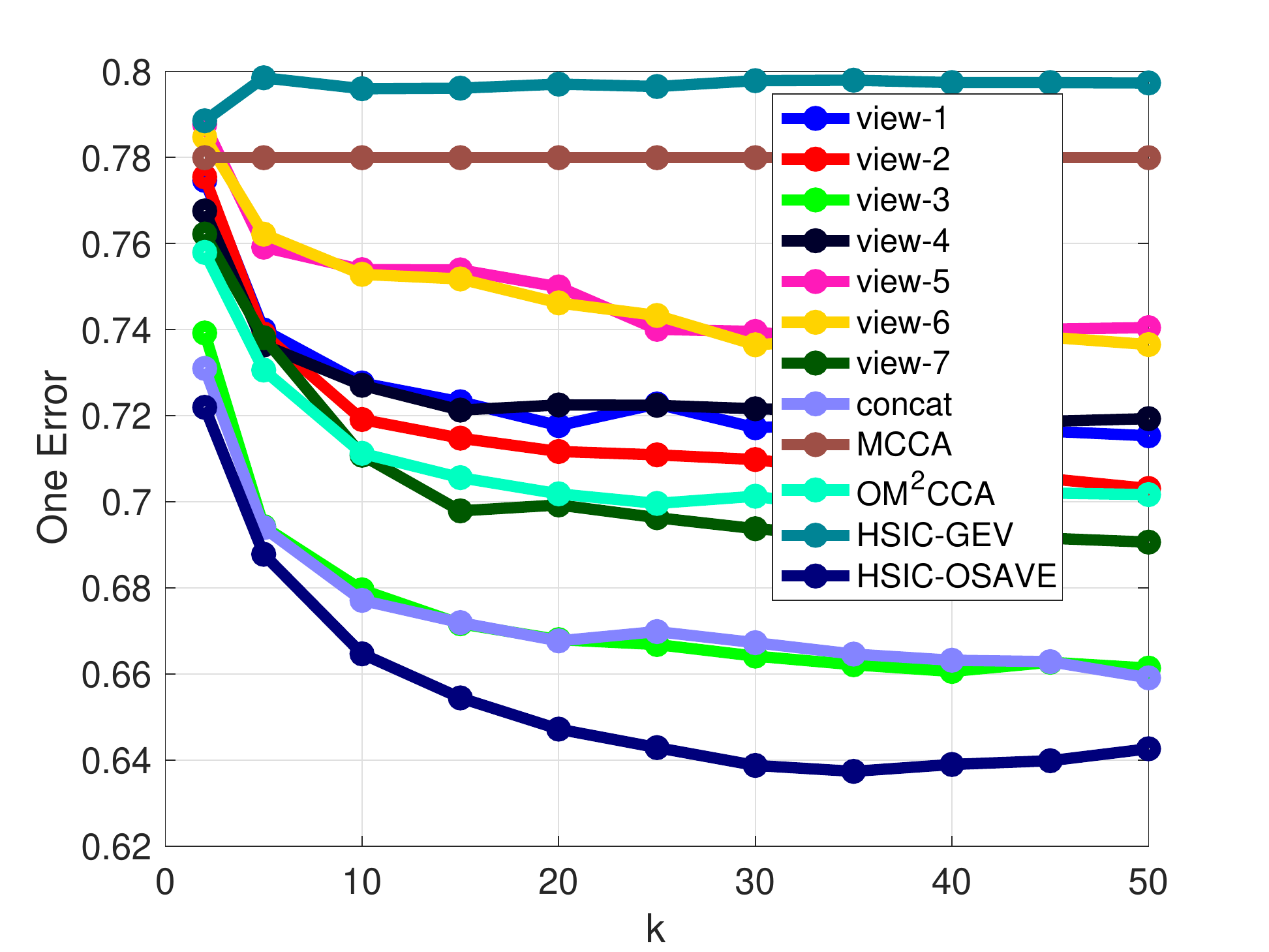} & 		
		\includegraphics[width=0.4\textwidth]{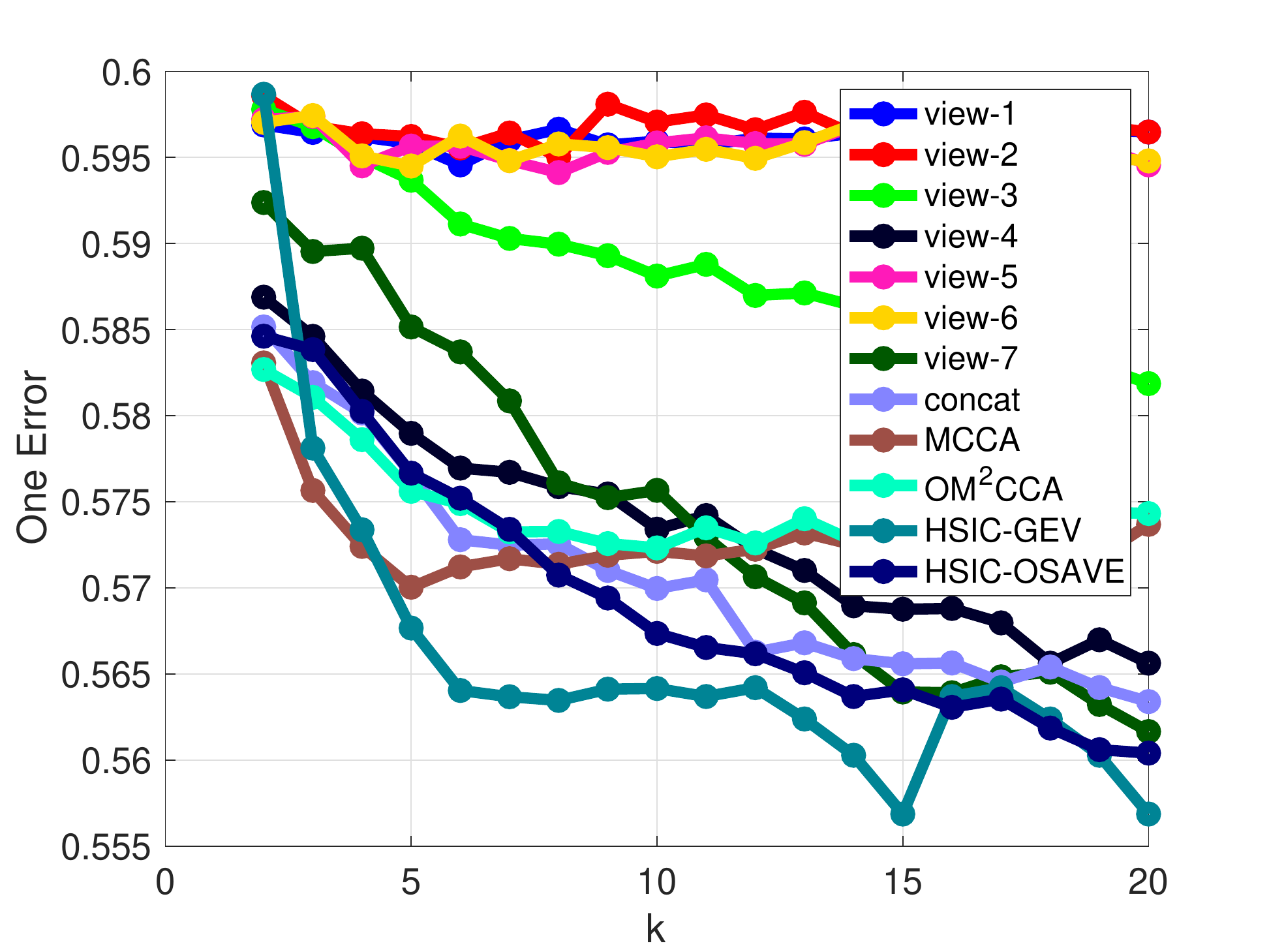}\\
		\includegraphics[width=0.4\textwidth]{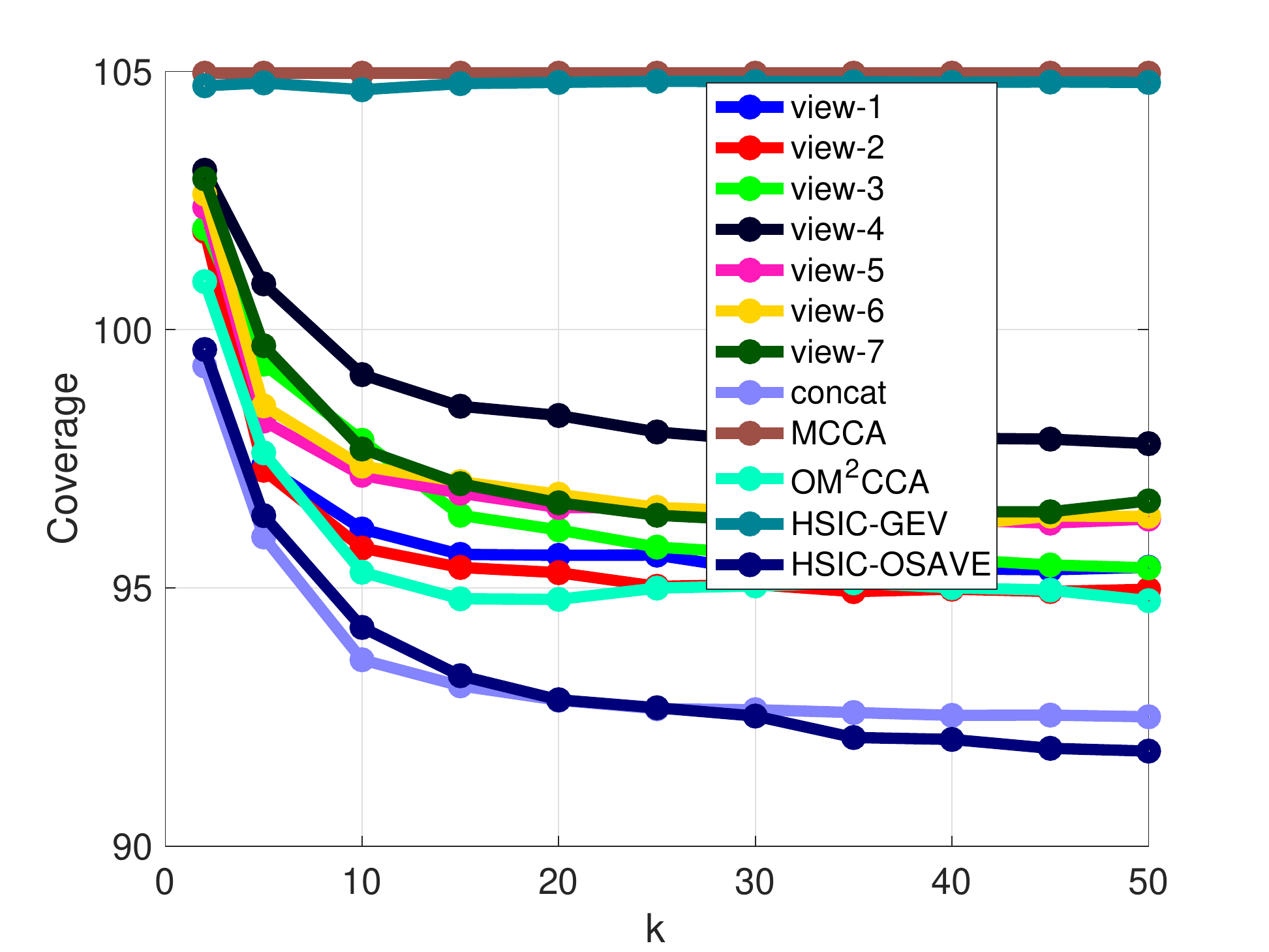} & 		\includegraphics[width=0.4\textwidth]{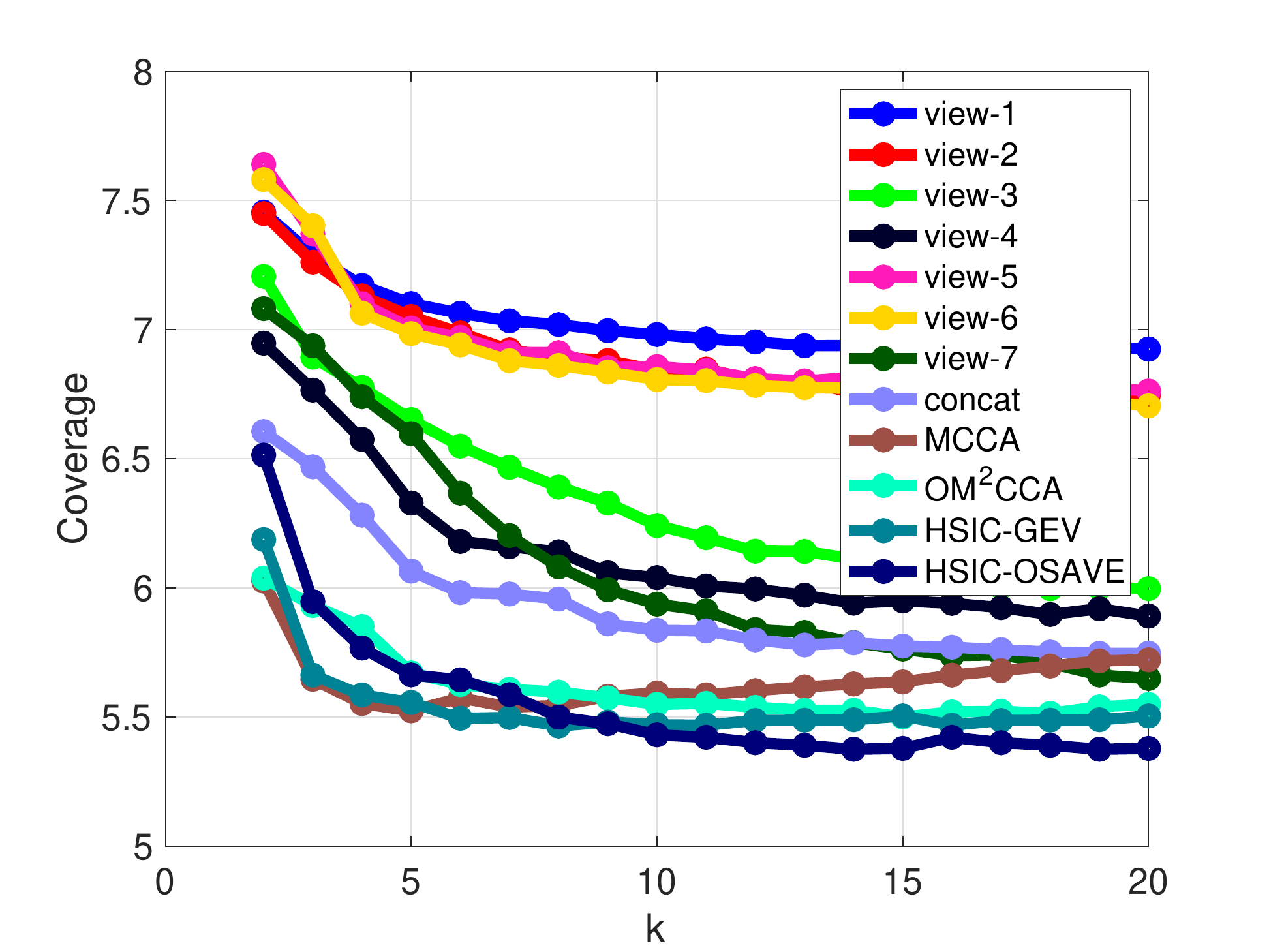}\\
		\includegraphics[width=0.4\textwidth]{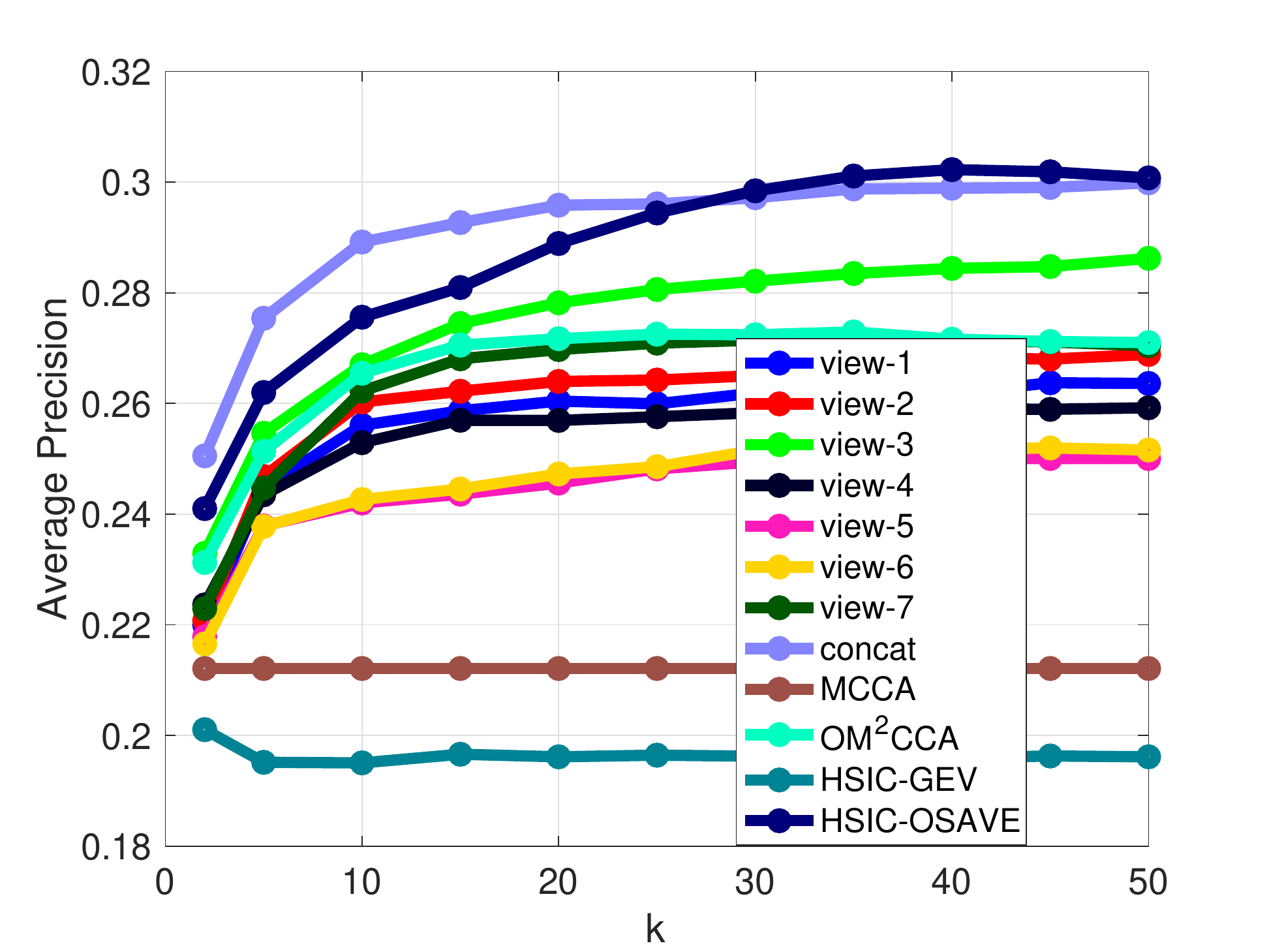} & 		
		\includegraphics[width=0.4\textwidth]{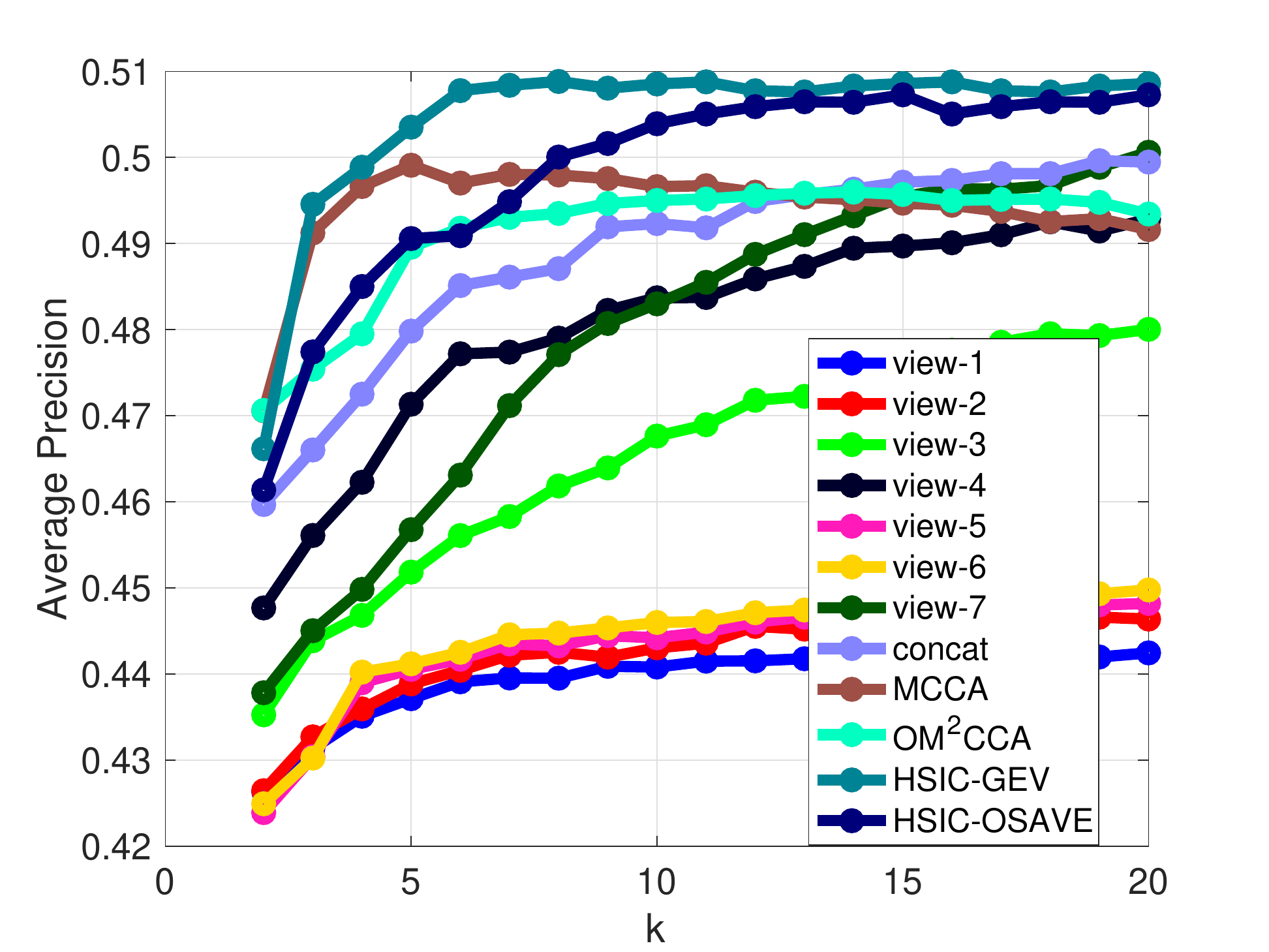}
	\end{tabular}
	\caption{Results with respect to five metrics by compared methods  on Corel5k (top row) and pascal07 (bottom row) over 10 random splits (10\% training and 90\%  testing), as  $k$ varies. } \label{fig:mvml-k}
\end{figure*}

\subsubsection{Datasets}

The statistics of six publicly available datasets are shown in Table~\ref{tab:mvml}, which are employed to evaluate the proposed methods for multi-view multi-label classification. Dataset emotions\footnote{http://mulan.sourceforge.net}  has two feature views: 8 rhythmic attributes and 64 timbre attributes. Corel5k \cite{duygulu2002object} is a benchmark dataset for keyword based image retrieval and image annotation. Dataset espgame \cite{makadia2008new} is obtained from an online game where two players gain points by agreeing on words describing the image. Dataset iaprtc12 \cite{makadia2008new} is a set of images accompanied with descriptions in several languages for cross-lingual retrieval. Both pascal07 \cite{everingham2010pascal} and mirflickr \cite{huiskes2008mir} are collected from the Flickr website. The last five datasets have been preprocessed with various feature descriptors and are  publicly available\footnote{http://lear.inrialpes.fr/people/guillaumin/data.php} \cite{guillaumin2009tagprop,guillaumin2010multimodal}. In our experiments, we choose $7$ descriptors: DenseHue (100), DenseHueV3H1 (300), DenseSift (1000), Gist (512), HarrisHue (100), HarrisHueV3H1 (300), and HarrisSift (1000).

\subsubsection{Compared Methods}
Multi-label classification \cite{tsoumakas2007multi} is a variant of the classification problem, where one instance
may have various numbers of labels from a set of predefined categories, i.e., a subset of labels. In addition, the multi-view multi-label classification data consists of multiple views as the input.
It is different from multi-view feature extraction in Section \ref{sec:mvfe-exp}, where each instance only has a single label.

Following \cite{wang2020orthogonal,Zhang2020Self}, we first use a multi-view subspace learning method as a supervised dimensionality reduction step for the purpose multi-view multi-label classification
so that the embeddings obtained by the method hopefully encode important correlations among multiple views and their output labels, and then  multi-label classification is conducted in the common space.
Hence, it is expected to have better performance for multi-view multi-label classification comparing with a single-view method
applied to each view only or to the naive concatenation approach. Specifically, we compare the following multi-view subspace learning approaches:
\begin{itemize}
	\item view-$s$: PCA on the $s$th view.
	\item concat: the concatenation of all views in the common space by PCA.
	\item MCCA \cite{nielsen2002multiset}: the output labels  considered as an additional view. Hence, there are $v+1$ views. The projection matrix for the output labels is learned but not used.
	\item HSIC-GEV: the proposed model solved as a generalized eigenvalue problem, which is similar to MLDA, but $\Phi_{s,s}$ is defined in (\ref{eq:hsic-1}) catering for multi-label outputs.
	\item OM$^2$CCA: the proposed model instantiated from (\ref{op:ouf}) for $v+1$ views with
	(\ref{eq:omcca}). Different from  \cite{Zhang2020Self}, all multiple views as  input are used.
	\item OHSIC: the proposed model instantiated from (\ref{op:ouf}) with
	(\ref{eq:hsic}).
\end{itemize}
After the projection matrices are learned,
we apply ML-kNN\footnote{http://lamda.nju.edu.cn/files/MLkNN.rar} in the common space as the backend multi-label classifier \cite{zhang2007ml}, which has demonstrated good performance over various datasets.

\subsubsection{Performance Evaluation}
Five widely-used metrics are used to measure the performance, including Hamming Loss, One Error, Ranking Loss, Coverage and Average Precision. Each evaluates the performance of a multi-label predictor from different aspects. Their concrete definitions can be found in \cite{zhang2013review}. In particular, the larger the Average Precision is, the better the performance, while for the other four metrics, the smaller the value the better the performance. Following \cite{zhang2007ml}, for each method we report the best results and their standard deviations over $10$ random training/testing splits in each of the five metrics.

Results by compared methods  are shown in Table~\ref{tab:mvml-metric}, in which the best results are reported by tuning $\alpha \in \{0.01, 0.1, 1, 10, 100\}$ and $k \in \{2, 5:5:50\}$ except for emotions, mirflickr and pascal07 (MCCA and OM$^2$CCA cannot have $k$ larger than the number of labels), over $10$ random splits of $10\%$ training and $90\%$ testing. From Table~\ref{tab:mvml-metric}, it can be observed that  (i) the joint subspace learning methods generally work better than PCA and the concatenation of individually projected views by PCA, (ii) the proposed HSIC-GEV and OHSIC consistently outperform others except in terms of Ranking Loss on emotions, Corel5k and espgame, and (iii) HSIC-GEV takes the top spots on iaprtc12 and mirflickr, while OHSIC takes most of the top spots on emotions, Corel5k and espgame. On pascal07, both approaches work equally well.

We further investigate the impact of parameter $k$ on each of the five metrics. Fig. \ref{fig:mvml-k} shows the trends of five metrics on Corel5k and pascal07 as $k$ varies. It is observed that a large $k$ generally leads to better performance for all methods, as it should be. Although Hamming Loss on Corel5k shows some fluctuation, the absolute difference is negligibly in the order of $10^{-5}$. In summary, HSIC-GEV and OHSIC can work consistently well over all tested $k$s.

\section{Conclusion}\label{sec:conclusions}
In this paper, we start by proposing a unified multi-view subspace learning framework, which aims to learn a set of orthogonal projections for desirable advantages such as more noise-tolerant, better suited for data visualization and distance preservation.
The proposed framework can be easily extended for single-view and multi-view learning in the settings of both unsupervised and supervised learning.
An efficient successive approximations via eigenvectors method (OSAVE) is designed to approximately solve the  optimization problem resulted from the proposed framework.
It is built upon well developed numerical linear algebra technique and can handle large scale datasets.
To verify the capability of the proposed framework and the approximate optimization method, we showcases six new models for two learning tasks. Experimental results  on various real-world datasets demonstrate that our proposed models solved
by our successive approximation method OSAVE perform competitively to and often better than the baselines.

\end{document}